\documentclass[final,5p,times,twocolumn]{elsarticle}

\usepackage{lineno,hyperref}
\modulolinenumbers[1]
\journal{Pattern Recognition Letters}

\usepackage[utf8]{inputenc} 
\usepackage[T1]{fontenc}    
\usepackage[]{hyperref}       

\usepackage{booktabs}       
\usepackage{amsfonts}       
\usepackage{nicefrac}       
\usepackage{microtype}      
\usepackage{xcolor}         

\usepackage{graphicx}
\usepackage{subfigure}
\usepackage{multirow}
\usepackage[export]{adjustbox}
\usepackage{enumerate}
\usepackage{amsmath,amssymb,amsthm,mathtools}
\usepackage{caption}
\usepackage{dsfont}
\usepackage{textcomp}
\usepackage[multiple]{footmisc}
\usepackage{algorithm,algorithmic}
\usepackage{wrapfig}
\usepackage{enumitem}
\usepackage{grffile}
\usepackage{bm}
\usepackage{xurl}            

\newcommand\numberthis{\addtocounter{equation}{1}\tag{\theequation}}
\DeclareMathOperator*{\argmin}{arg\,min}

\theoremstyle{definition}
\newtheorem{definition}{Definition}
\newtheorem*{definition*}{Definition}

\newtheorem*{example*}{Example}
\theoremstyle{theorem}
\newtheorem{theorem}[definition]{Theorem}

\newtheorem*{theorem*}{Theorem}
\newtheorem*{lemma*}{Lemma}
\newtheorem*{proposition*}{Proposition}
\newtheorem*{corollary*}{Corollary}
\newtheorem*{remark*}{Remark}
\newtheorem*{claim*}{Claim}
\newtheorem*{problem*}{Problem}
\newtheorem*{observation*}{Observation}
\newcommand*\diff{\mathop{}\!\mathrm{d}}

\allowdisplaybreaks

\begin{document}

\begin{frontmatter}

\title{Neural Posterior Regularization for Likelihood-Free Inference}

\author[add1]{Dongjun Kim}
\author[add2]{Kyungwoo Song}
\author[add1]{Seungjae Shin}
\author[add3]{Wanmo Kang}
\author[add1]{Il-Chul Moon}
\author[add4]{Weonyoung Joo\corref{cor}}
\ead{weonyoungjoo@ewha.ac.kr}

\address[add1]{Department of Industrial and Systems Engineering, Korea Advanced Institute of Science and Technology (KAIST), Daejeon, Republic of Korea}
\address[add2]{Department of Artificial Intelligence, University of Seoul, Seoul, Republic of Korea}
\address[add3]{Department of Mathematical Sciences, Korea Advanced Institute of Science and Technology (KAIST), Daejeon, Republic of Korea}
\address[add4]{Department of Statistics, EWHA Womans University, Seoul, Republic of Korea}
\cortext[cor]{Corresponding author}

\begin{abstract}
A simulation is useful when the phenomenon of interest is either expensive to regenerate or irreproducible with the same context. 
Recently, Bayesian inference on the distribution of the simulation input parameter has been implemented sequentially to minimize the required simulation budget for the task of simulation validation to the real-world. 
However, the Bayesian inference is still challenging when the ground-truth posterior is multi-modal with a high-dimensional simulation output. 
This paper introduces a regularization technique, namely Neural Posterior Regularization (NPR), which enforces the model to explore the input parameter space effectively. 
Afterward, we provide the closed-form solution of the regularized optimization that enables analyzing the effect of the regularization. 
We empirically validate that NPR attains the statistically significant gain on benchmark performances for diverse simulation tasks. 
\end{abstract}


\end{frontmatter}


\section{Introduction}

Recently, enhanced computing power has motivated the construction of highly complex simulation models. However, these high-resolution simulations achieve high precision only if we calibrate the adjustable simulation input parameters because otherwise, the simulation outcome significantly varies from the real world. The Bayesian inference \cite{papamakarios2019sequential, greenberg2019automatic} is one possible yet general framework for such calibration task. The Bayesian inference, however, is challenging because the likelihood function of the simulation is intractable \citep{papamakarios2016fast} in general.

\begin{figure}[t]
    \centering
    \subfigure[True]{\includegraphics[width=.32\linewidth]{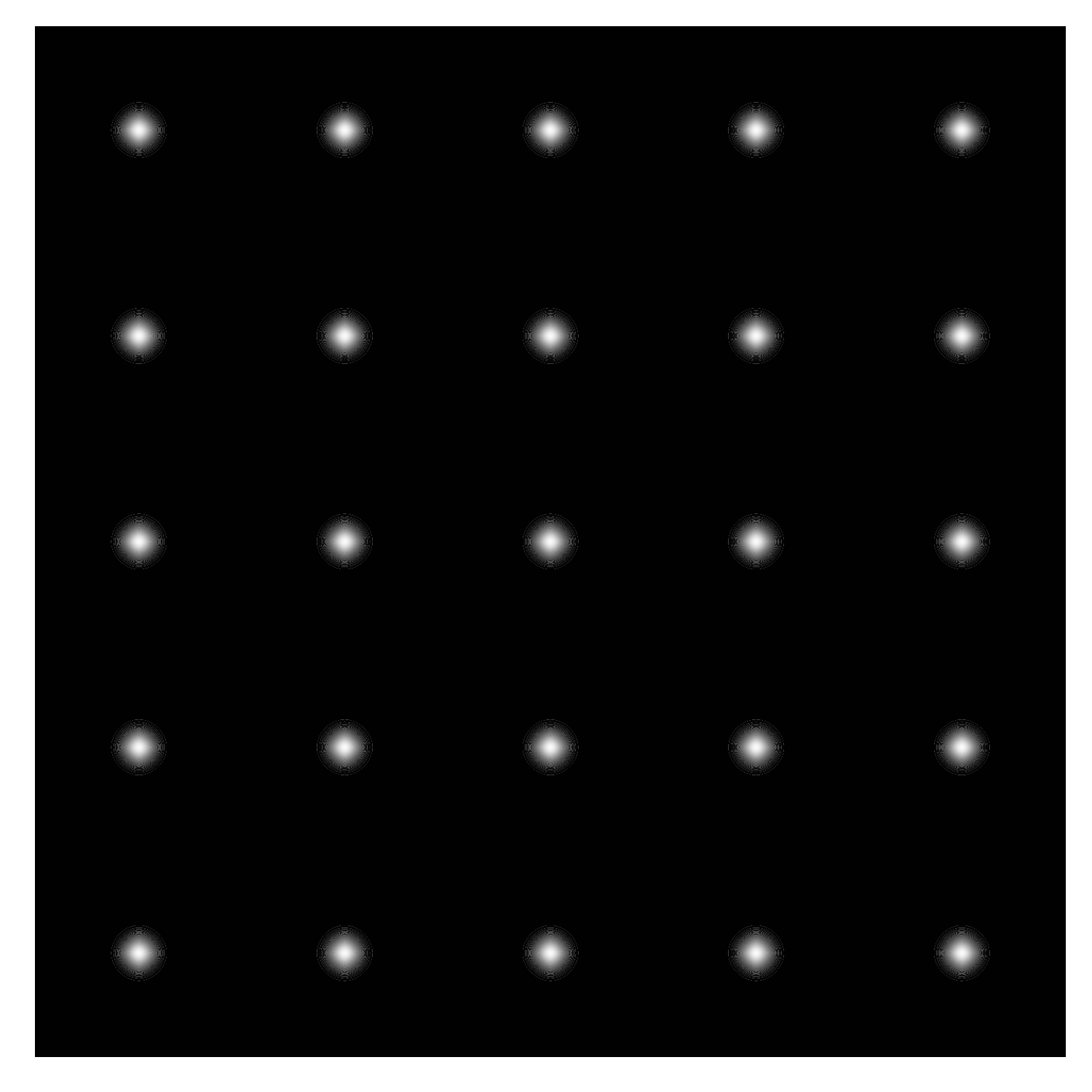}}
    \centering
    \subfigure[SNL]{\includegraphics[width=.32\linewidth]{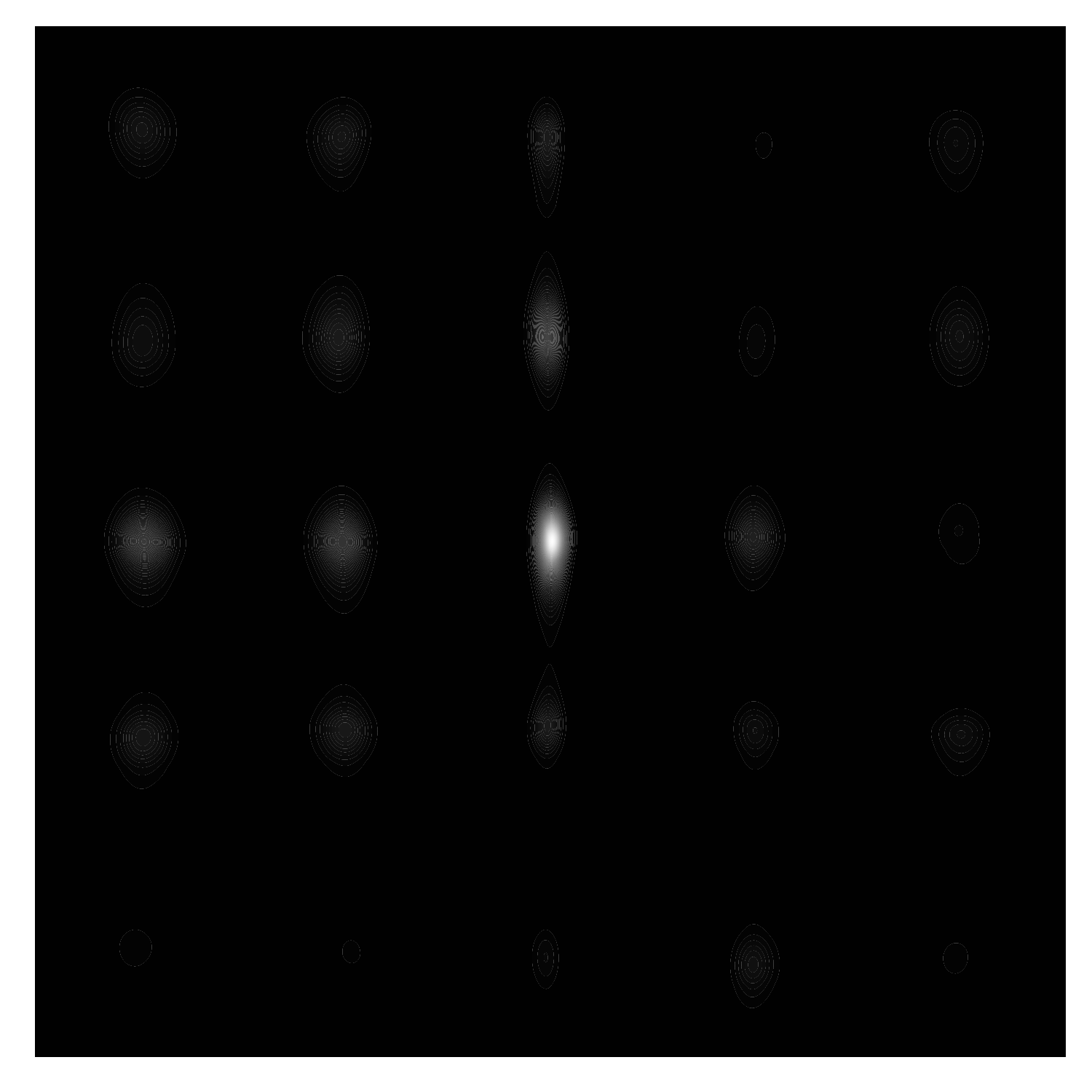}}
    \centering
    \subfigure[SNL+NPR]{\includegraphics[width=.32\linewidth]{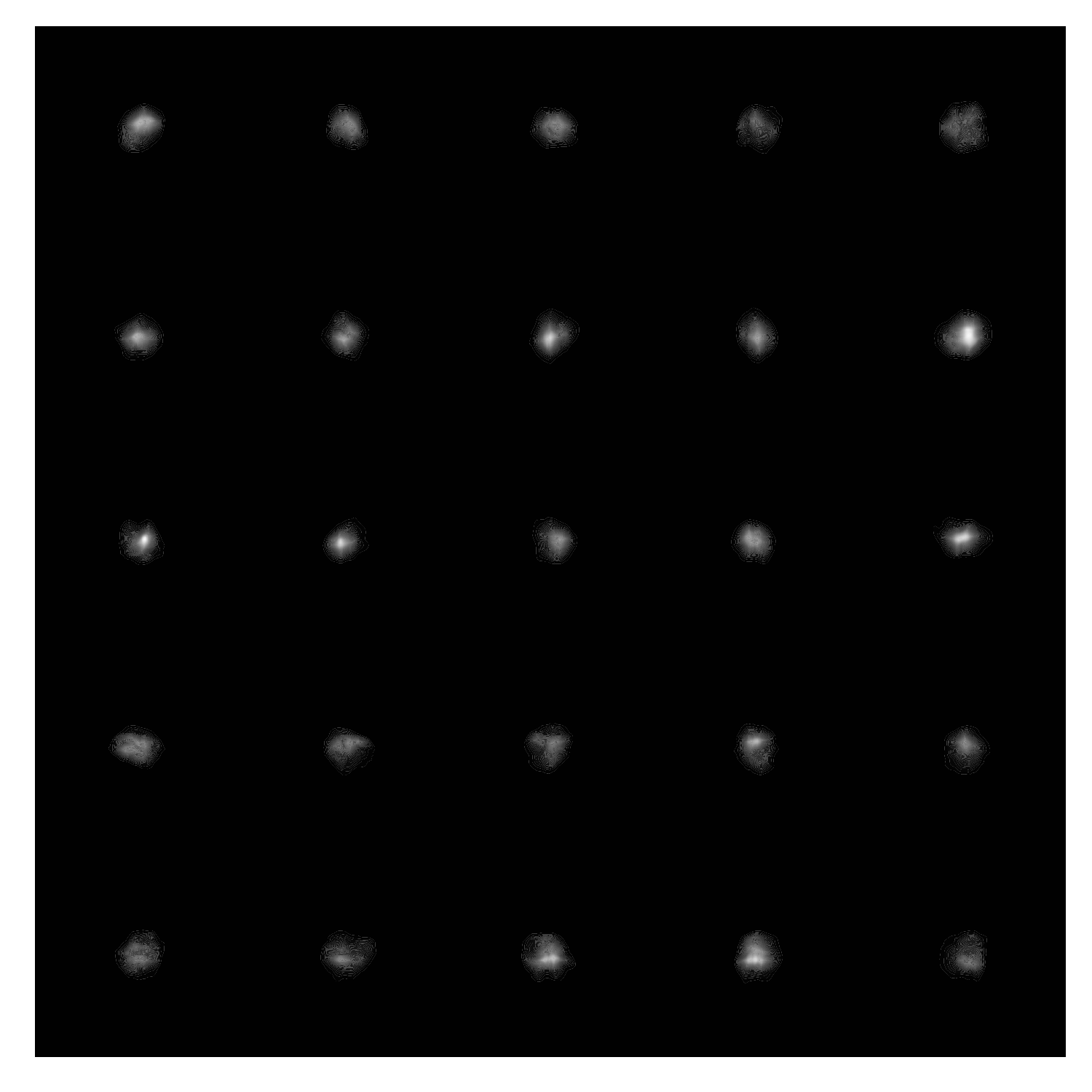}}
    \vskip -0.05in
    \caption{Comparison of (a) the true posterior and the approximate posteriors inferred by (b) SNL and (c) SNL+NPR. We select $(\theta_{1},\theta_{2})$ from a uniform distribution on $[0,1]^{2}$ and synthetically generate the outcome from a Gaussian distribution $\mathcal{N}([\cos{5\pi\theta_{1}},\cos{5\pi\theta_{2}}],0.1\mathbf{I})$ with $\mathbf{x}_{o}=[0,0]$.}
    \label{fig:comparison_2d}
    \vskip -0.1in
\end{figure}

\begin{figure*}[t]
    \centering
    \subfigure[Overfitting]{\includegraphics[width=.32\linewidth]{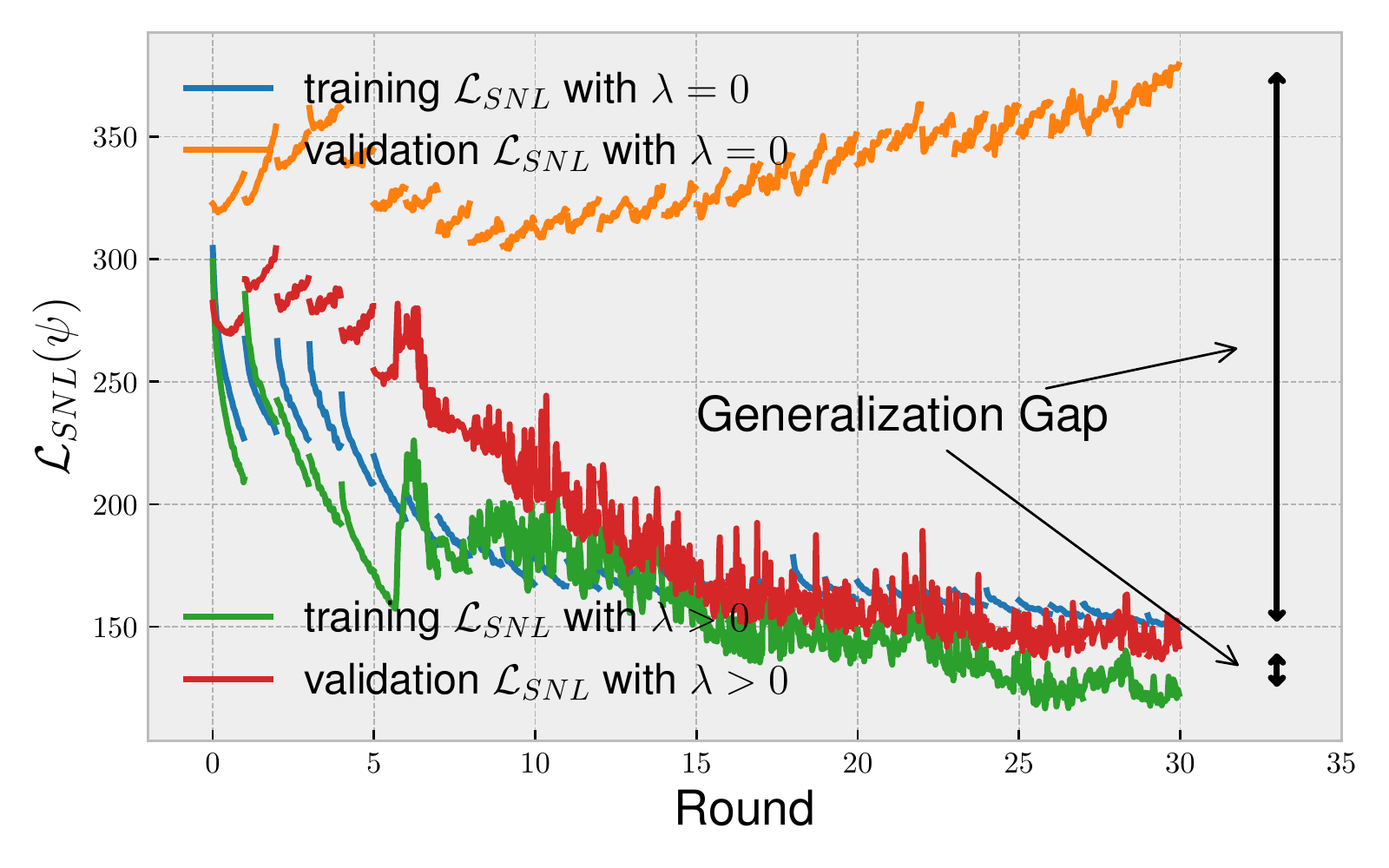}}
    \centering
    \subfigure[Poor Inference]{\includegraphics[width=.32\linewidth]{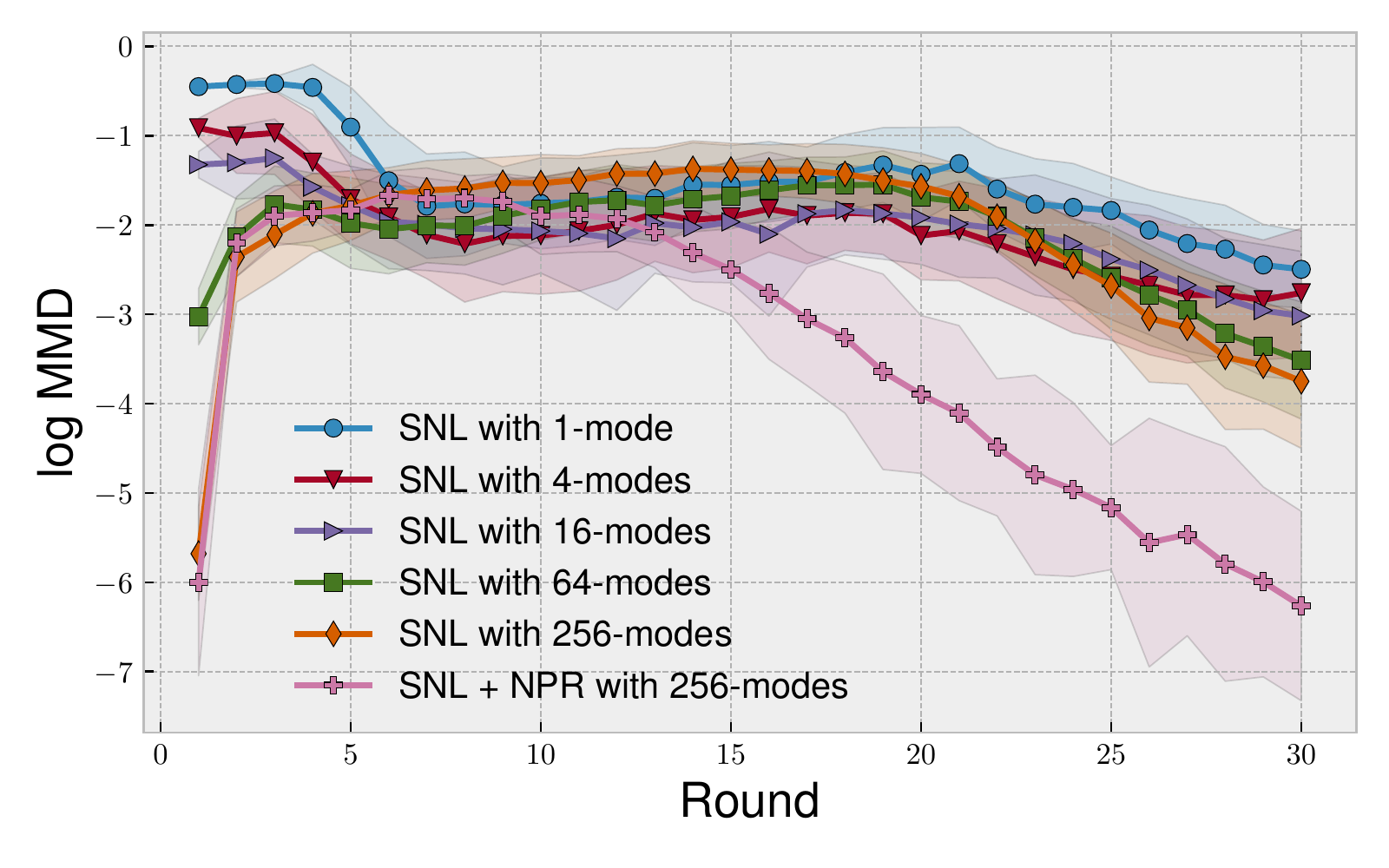}}
    \centering
    \subfigure[Non-scalability]{\includegraphics[width=.32\linewidth]{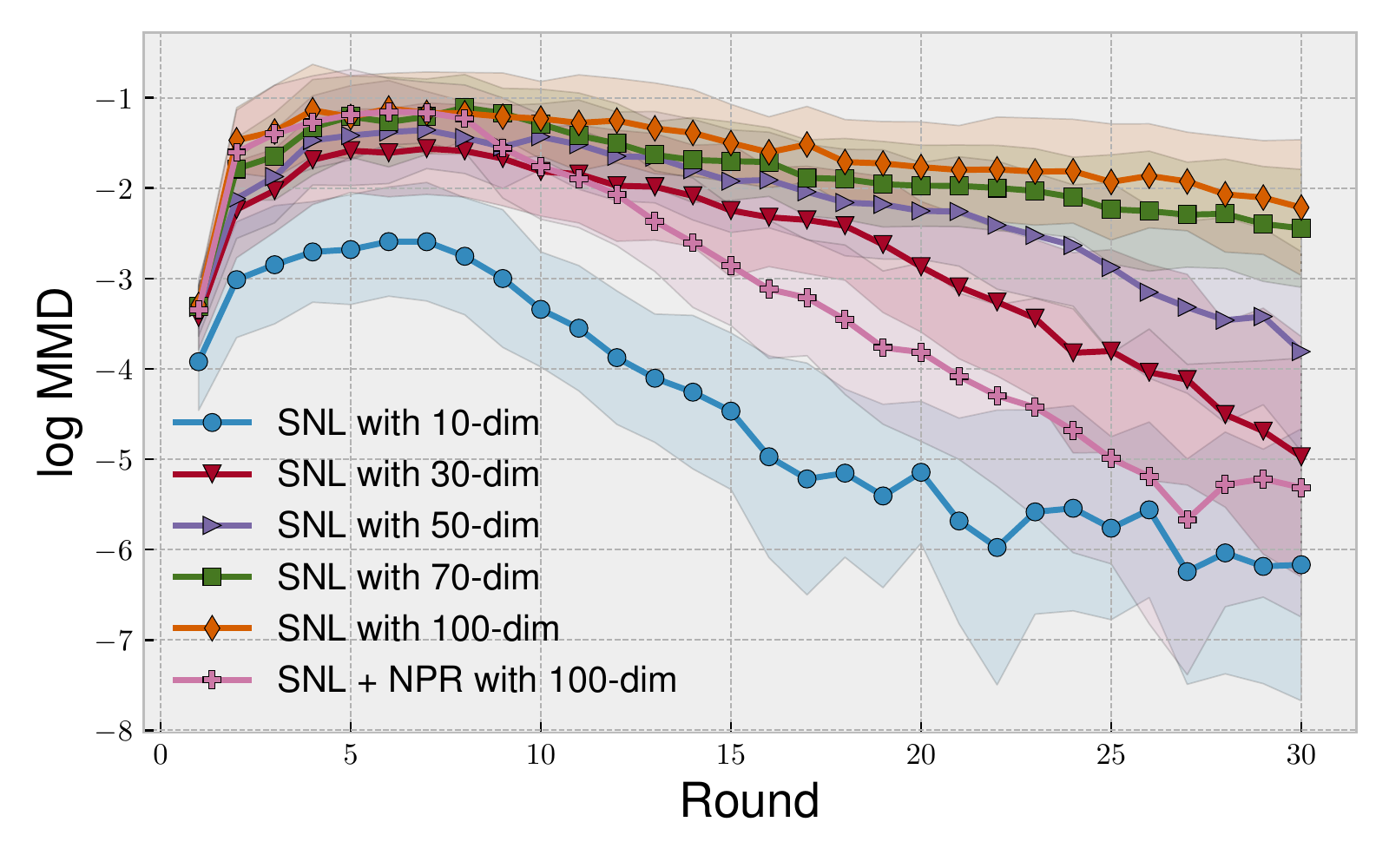}}
    \vskip -0.05in
    \caption{SNL with no regularization ($\lambda=0$) on SLCP-16 \citep{kim2020sequential} suffers from (a) large generalization gap (b) inaccurate posterior estimation on multi-modal posterior (c) slow learning on high dimensional output.}
    \label{fig:SNL}
    \vskip -0.05in
\end{figure*}

Recent algorithms on the simulation-based inference, or \textit{likelihood-free inference} \cite{papamakarios2016fast}, parametrizes the model posterior distribution with a neural network. However, not every neural network-based \textit{likelihood-free inference} is advantageous over previous practices in terms of the amount of simulation budget. Besides, as the simulation complexity increases, the posterior distribution is likely to attain multi-modalities \citep{aushev2020likelihood}, and these factors combine to lead the inference to a highly challenging task. For instance, Fig. \ref{fig:comparison_2d} illustrates that the inference fails with limited data on a toy simulation with a multi-modal posterior.

We introduce a regularization technique to solve this multi-modal issue, particularly in a high-dimensional case \cite{kim2020sequential}. This regularization is a mixture of the reverse KL divergence and the mutual information that behaves in comparative ways. In the middle rounds of \textit{likelihood-free inference}, the reverse KL regularization encourages the mode-centered dataset because it has the mode-seeking property; and the mutual information regularization estimates the complex posterior within the budget since it captures the rich representation. As the regularization is intractable in its vanilla form, we provide the proxy of the regularization that is tractable, namely Neural Posterior Regularization (NPR). Afterward, we analyze this NPR theoretically by providing the closed-form solution of the regularized loss. Also, we empirically demonstrate its advances under the multi-modal and high-dimensional settings of simulations.

The remainder of this paper is organized as follows. In Section \ref{sec:PR}, we present the preliminaries of the \textit{likelihood-free inference} and its hurdles. To overcome such difficulties of the current practice, we propose a new regularization in Section \ref{sec:NPR}, and its detailed estimation together with theorems in Section \ref{sec:Methodology}. Motivated by the theorems, we suggest an efficient scheduling method for the regularization coefficient in Section \ref{sec:regularization_coefficient}. Empirical results in various simulations are presented in Section \ref{sec:Experiments}. 

\section{Preliminary}\label{sec:PR}

\subsection{Problem Definition}

The evaluation on the likelihood of $p_{sim}(\mathbf{x}|\bm{\theta})$ in a simulation is not allowed because a simulation is fundamentally a data-generation descriptive process. The purpose of \textit{likelihood-free inference} is estimating the posterior distribution $p_{sim}(\bm{\theta}|\mathbf{x}=\mathbf{x}_{o})$, where $\mathbf{x}_{o}$ is a one-shot real world observation, assuming that the real world with the same context happens only \textit{once}.

\subsection{Sequential Likelihood-Free Inference}

Recent approaches of \textit{likelihood-free inference} estimate the posterior with the iterative rounds of inference, where the iterative rounds gradually fasten the approximate posterior to the ground-truth posterior. This \textit{sequential} approach is becoming the mainstream in the community of \textit{likelihood-free inference} because the iterative optimization saves the simulation budgets by orders of magnitude \citep{cranmer2020frontier}.

\paragraph{Initial round} The sequential likelihood-free inference gathers simulation inputs from the prior distribution $p(\bm{\theta})$, the uniform distribution on the input space. A collection of simulation input-output pairs constructs a dataset for the initial round $\mathcal{D}_{1}=\{(\bm{\theta}_{1,j},\mathbf{x}_{1,j})\}_{j=1}^{N}$, where each $\mathbf{x}_{1,j}$ is the simulation output corresponding to $\bm{\theta}_{1,j}$, i.e., $\mathbf{x}_{1,j}\sim p_{sim}(\mathbf{x}\vert\bm{\theta}_{1,j})$. The approximate posterior on the initial round with $\mathcal{D}_{1}$ is trained by either one of the inference algorithms described in Section \ref{sec:inference}.

\paragraph{Next rounds} The new simulation inputs are drawn from a proposal distribution $\bm{\theta}_{r,j}\sim p_{r}(\bm{\theta})$. The proposal distribution is the approximate posterior at the last round $p_{r}(\bm{\theta}):=q_{r-1}(\bm{\theta}\vert\mathbf{x}_{o})$. The algorithm accumulates the newly simulated data into the training dataset as $\mathcal{D}_{r}\leftarrow\mathcal{D}_{r-1}\cup\{(\bm{\theta}_{r,j},\mathbf{x}_{r,j})\}_{j=1}^{N}$, where $\mathbf{x}_{r,j}\sim p(\mathbf{x}\vert\bm{\theta}_{r,j})$, and approximates the posterior.

\subsection{Inference Algorithms}\label{sec:inference}

\paragraph{Neural Likelihood} \citet{papamakarios2019sequential} estimates the likelihood $p(\mathbf{x}\vert\bm{\theta})$ with a (conditional) neural network $q_{\mathbf{\psi}}(\mathbf{x}\vert\bm{\theta})$ parametrized by $\mathbf{\psi}$. Given $\tilde{p}_{r}(\bm{\theta}):=\frac{1}{r}\sum_{s=1}^{r}p_{s}(\bm{\theta})$ is the cumulative input distribution, the optimization loss becomes
\begin{align*}
\mathcal{L}_{SNL}(\mathbf{\psi})&=-\mathbb{E}_{(\bm{\theta},\mathbf{x})\sim \tilde{p}_{r}(\bm{\theta})p_{sim}(\mathbf{x}\vert\bm{\theta})}\big[\log{q_{\mathbf{\psi}}(\mathbf{x}\vert\bm{\theta})}\big]\\
&=D_{KL}\big(\tilde{p}_{r}(\bm{\theta})p_{sim}(\mathbf{x}\vert\bm{\theta})\Vert\tilde{p}_{r}(\bm{\theta})q_{\mathbf{\psi}}(\mathbf{x}\vert\bm{\theta})\big)+C,
\end{align*}
where $C$ is a constant irrelevant to $\psi$. The optimal neural likelihood $q_{\psi^{*}}(\mathbf{x}\vert\bm{\theta})$ matches to the ground-truth likelihood $p_{sim}(\mathbf{x}\vert\bm{\theta})$ if the training dataset sufficiently covers the space of input parameters. The approximate posterior of this Sequential Neural Likelihood (SNL) is given as the unnormalized form by $q_{r}(\bm{\theta}\vert\mathbf{x}_{o})\propto p(\bm{\theta})q_{\mathbf{\psi}}(\mathbf{x}_{o}\vert\bm{\theta})$ with the trained $\mathbf{\psi}$.

\paragraph{Neural Posterior} \citet{greenberg2019automatic} directly estimates the posterior $p_{sim}(\bm{\theta}\vert\mathbf{x})$ with a neural network $q_{\mathbf{\phi}}(\bm{\theta}\vert\mathbf{x})$ parametrized by $\mathbf{\phi}$. The optimization loss of this Automatic Posterior Transformation (APT) is the (normalized) negative log-posterior $\mathcal{L}_{APT}(\mathbf{\phi})=-\mathbb{E}_{(\bm{\theta},\mathbf{x})\sim \tilde{p}_{r}(\bm{\theta})p_{sim}(\mathbf{x}\vert\bm{\theta})}\big[\log{\frac{q_{\mathbf{\phi}}(\bm{\theta}\vert\mathbf{x})}{Z_{\mathbf{\phi}}(\mathbf{x})}}\big]$, which is equivalent to the KL divergence $D_{KL}\big(\tilde{p}_{r}(\bm{\theta})p_{sim}(\mathbf{x}\vert\bm{\theta})\Vert \tilde{p}_{r}(\bm{\theta})\frac{q_{\mathbf{\phi}}(\bm{\theta}\vert\mathbf{x})}{p(\bm{\theta})}\frac{\tilde{p}_{r}(\mathbf{x})}{Z_{\mathbf{\phi}}(\mathbf{x})}\big)$. Here, $Z_{\mathbf{\phi}}(\mathbf{x})=\int q_{\mathbf{\phi}}(\bm{\theta}\vert\mathbf{x})\frac{\tilde{p}_{r}(\bm{\theta})}{p(\bm{\theta})}\diff\bm{\theta}$ is the normalizing constant and $\tilde{p}_{r}(\mathbf{x})=\mathbb{E}_{\tilde{p}_{r}(\bm{\theta})}[p_{sim}(\mathbf{x}\vert\bm{\theta})]$ is the outcome distribution expected by $\tilde{p}_{r}(\bm{\theta})$.

\subsection{Issues of Neural Likelihood}\label{sec:SNL_problem}

Despite the equivalence of the optimal neural likelihood to the ground-truth likelihood, SNL is under severe overfitting, as evidenced in Fig. \ref{fig:SNL}-(a). This overfitting prevents the accurate estimation of the neural likelihood, and the approximate posterior $q_{\psi}(\bm{\theta}\vert\mathbf{x}_{o})\propto p(\bm{\theta})q_{\mathbf{\psi}}(\mathbf{x}_{o}\vert\bm{\theta})$ becomes inaccurate in Fig. \ref{fig:SNL}-(b) and Fig. \ref{fig:SLCP-16_sample}-(b). This immatured inference at each round gives an inaccurate signal when we sample the next batch of the simulation run, and this feedback loop eventually leads the inference failure within a limited simulation budget. Furthermore, SNL is not scalable to the outcome dimension in Fig. \ref{fig:SNL}-(c).

\begin{figure*}[t]
\centering
    \subfigure[$p_{sim}(\bm{\theta}\vert\mathbf{x}_{o})$ (Ground-truth)]{\includegraphics[width=0.24\linewidth]{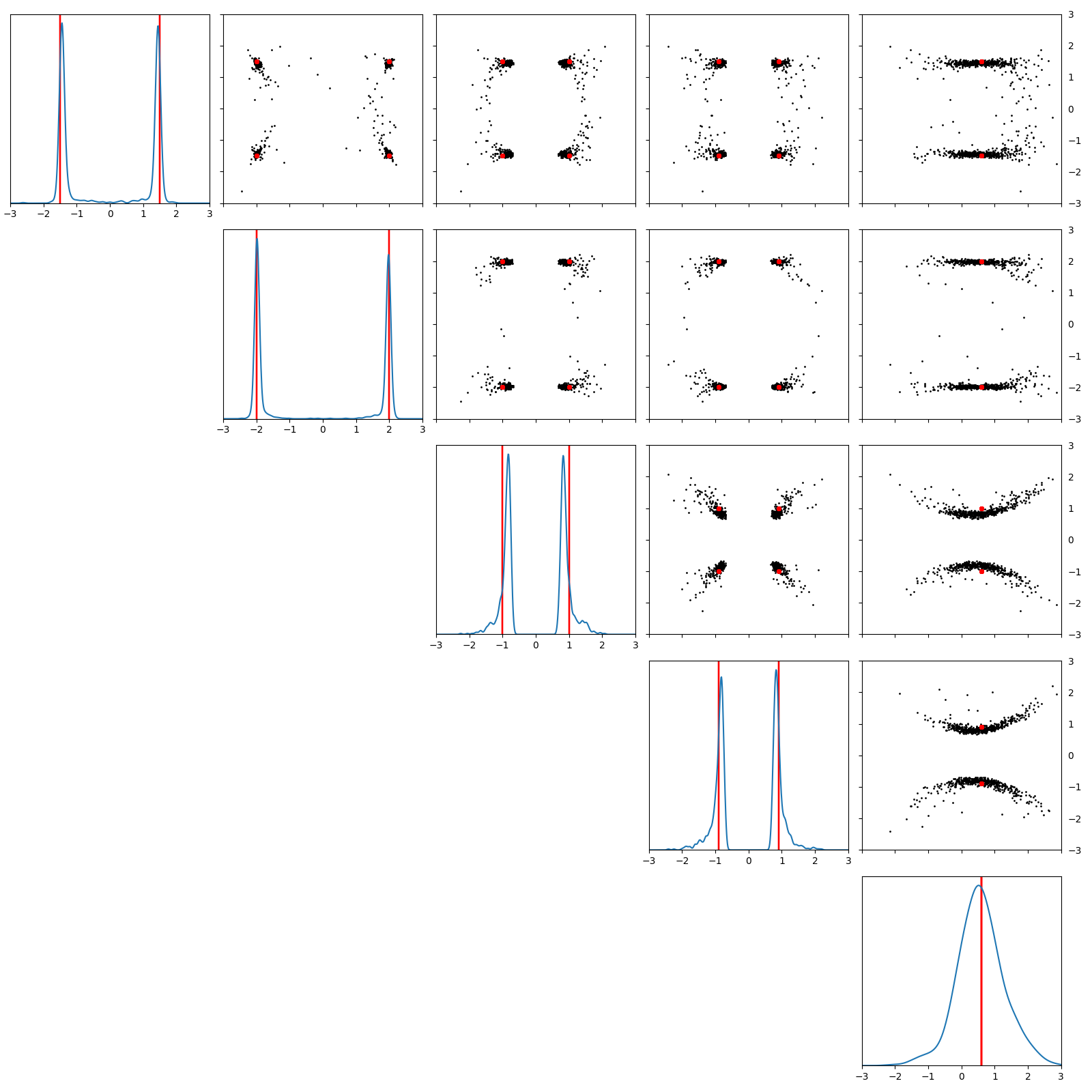}}
    \centering
    \subfigure[$\mathcal{L}_{SNL}(\psi)$]{\includegraphics[width=0.24\linewidth]{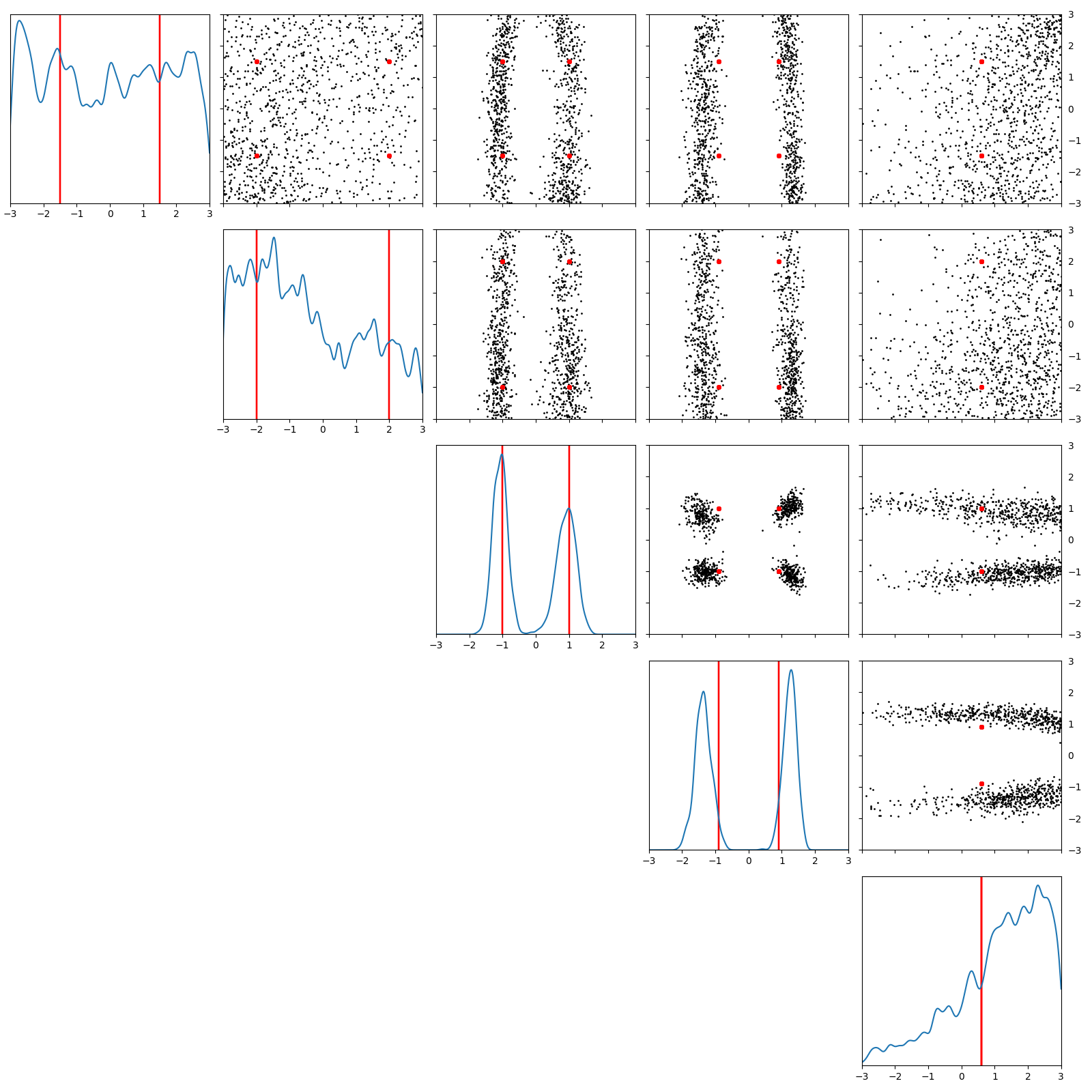}}
        \centering
    \subfigure[$\mathcal{L}_{SNL}(\psi)+F_{1}(\psi)$]{\includegraphics[width=0.24\linewidth]{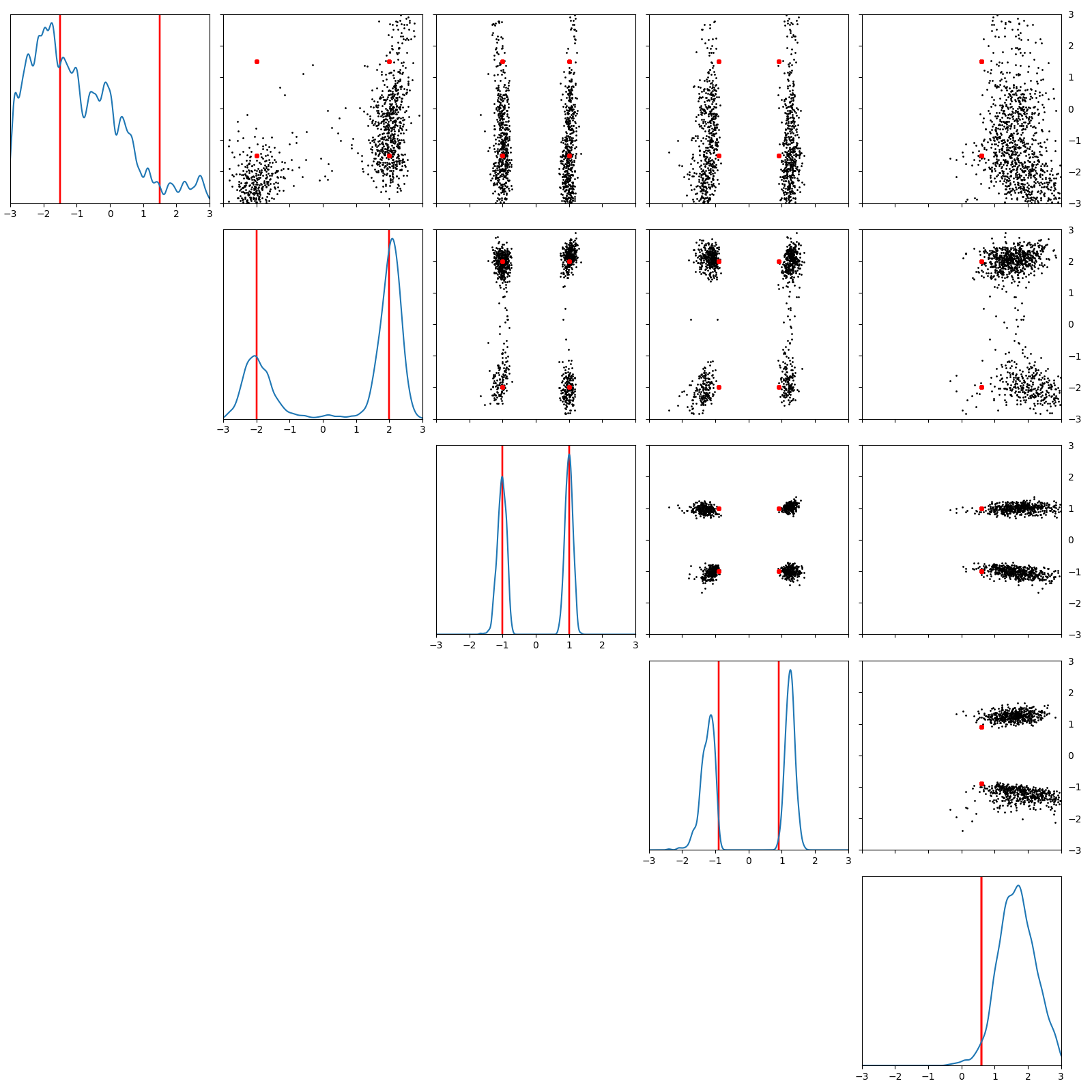}}
        \centering
    \subfigure[$\mathcal{L}_{SNL}(\psi)+\lambda (F_{1}(\psi)+F_{2}(\psi))$]{\includegraphics[width=0.24\linewidth]{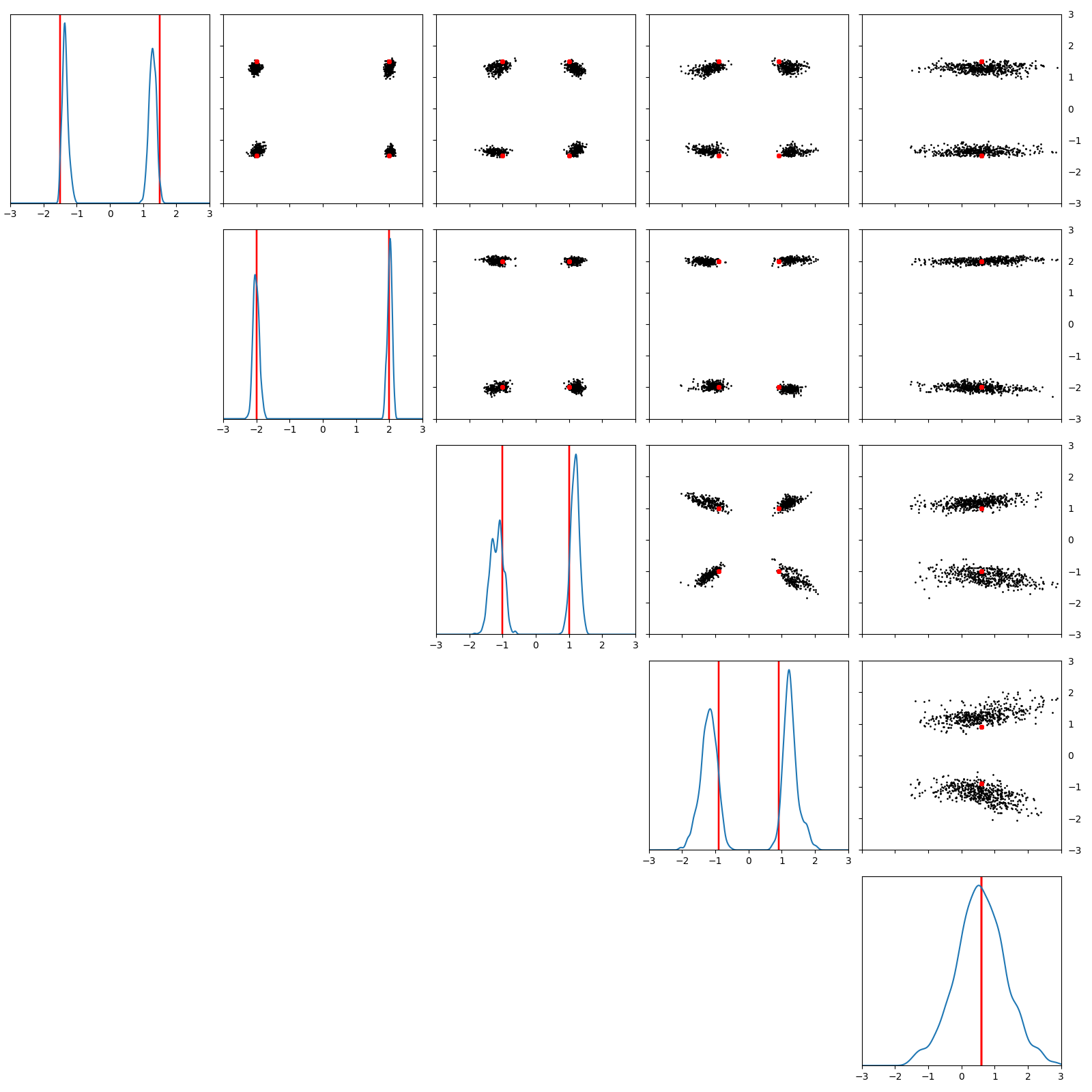}}
    \vskip -0.05in
    \caption{Comparison of (a) the ground-truth posterior $p_{sim}(\bm{\theta}\vert\mathbf{x}_{o})$ and the approximate posterior $q_{\psi}(\bm{\theta}\vert\mathbf{x}_{o})$ inferred by (b) SNL loss $\mathcal{L}_{SNL}(\psi)$, (c) SNL with the reverse KL regularization $\mathcal{L}_{SNL}(\psi)+\lambda F_{1}(\psi)$, and (d) SNL with our regularization $\mathcal{L}_{SNL}(\psi)+\lambda (F_{1}(\psi)+F_{2}(\psi))$ on SLCP-16 \cite{kim2020sequential}. The diagonal boxes represent marginal distributions for each parameter input, and the off-diagonal boxes represent the random samples from the posterior distribution as black dots, projected into the two-dimensional subspace for each pair of parameter inputs. The red dots and lines represent the ground-truth input parameters. We apply the adversarial method from \citet{nowozin2016f} to estimate $F_{1}$ in (c); we use the proposed $F(\psi,\phi)$ to estimate $F(\psi)$ in (d). We use the exponential decaying $\lambda$, see Section \ref{sec:regularization_coefficient} for the details.}
    \label{fig:SLCP-16_sample}
    \vskip -0.05in
\end{figure*}

\section{Regularization on Training the Neural Likelihood}\label{sec:NPR}

To resolve the aforementioned issues, we introduce a constrained problem
\begin{align}\label{eq:constrained_problem}
\begin{split}
&\text{minimize } \mathcal{L}_{SNL}(\mathbf{\psi}) \text{  subject to } F(\mathbf{\psi})\le C,
\end{split}
\end{align}
where $C$ is a constant that restricts the class of neural likelihood to $\mathcal{F}=\{q_{\mathbf{\psi}}:F(\mathbf{\psi})\le C\}$. The Lagrangian of the problem is
\begin{align*}
\mathcal{L}_{\lambda}(\mathbf{\psi})=\mathcal{L}_{SNL}(\mathbf{\psi})+\lambda F(\mathbf{\psi}),
\end{align*}
where $\lambda>0$ is the regularization magnitude. 

To determine a specific form of the constraint, recall that the optimization loss is the forward KL divergence between the true joint distribution $\tilde{p}_{r}(\bm{\theta})p_{sim}(\mathbf{x}\vert\bm{\theta})$ and the modeled joint distribution $\tilde{p}_{r}(\bm{\theta})q_{\mathbf{\psi}}(\mathbf{x}\vert\bm{\theta})$. Due to the mode-covering property \citep{zhang2019variational} of the forward KL divergence, the approximate posterior $q_{\psi}(\bm{\theta}\vert\mathbf{x}_{o})$ becomes inaccurate to the ground-truth posterior $p_{sim}(\bm{\theta}\vert\mathbf{x}_{o})$, as in Fig. \ref{fig:SLCP-16_sample}-(b) without penalizing the mode-covering. Specifically, we observe that $\bm{\theta}$ samples out of modes hardly contribute to the inference quality, so class $\mathcal{F}$ consisting of the mode concentrated distributions would have merits in the inference. In addition, the rich representation power would significantly mitigate the overfitting issue. Summing together, we design the constraint as
\begin{align*}
F(\mathbf{\psi})=F_{1}(\mathbf{\psi})+F_{2}(\mathbf{\psi}),
\end{align*}
where $F_{1}$ forces the posterior to exploit its modes and $F_{2}$ is for the better representation with the limited amount of data. 

As the reverse KL has the \textit{mode-seeking} property \citep{poole2016improved}, we define $F_{1}$ by
\begin{align*}
F_{1}(\mathbf{\psi})=D_{KL}\Big(\tilde{p}_{r}(\bm{\theta})q_{\mathbf{\psi}}(\mathbf{x}\vert\bm{\theta})\Big\Vert\tilde{p}_{r}(\bm{\theta})p_{sim}(\mathbf{x}\vert\bm{\theta})\Big).
\end{align*}
The mode-seeking property of the reverse KL strongly penalizes a dispersive distribution with non-zero values on the intermediate region between modes, whereas the forward KL $\mathcal{L}_{SNL}(\psi)$ prefers a mode-covering distribution. Therefore, the weighted loss $\mathcal{L}_{SNL}(\psi)+\lambda F_{1}(\psi)$ mixes two extremes in a unified optimization loss by taking both contrastive properties inherited from forward and reverse KLs. In other words, the weighted divergence searches distributions with accurate modes while retaining the mode diversity. The weight of $\lambda$ controls the trade-off between the exploration and exploitation effects. 

\begin{figure}[t]
\vskip -0.25in
\centering
    \includegraphics[width=.8\linewidth]{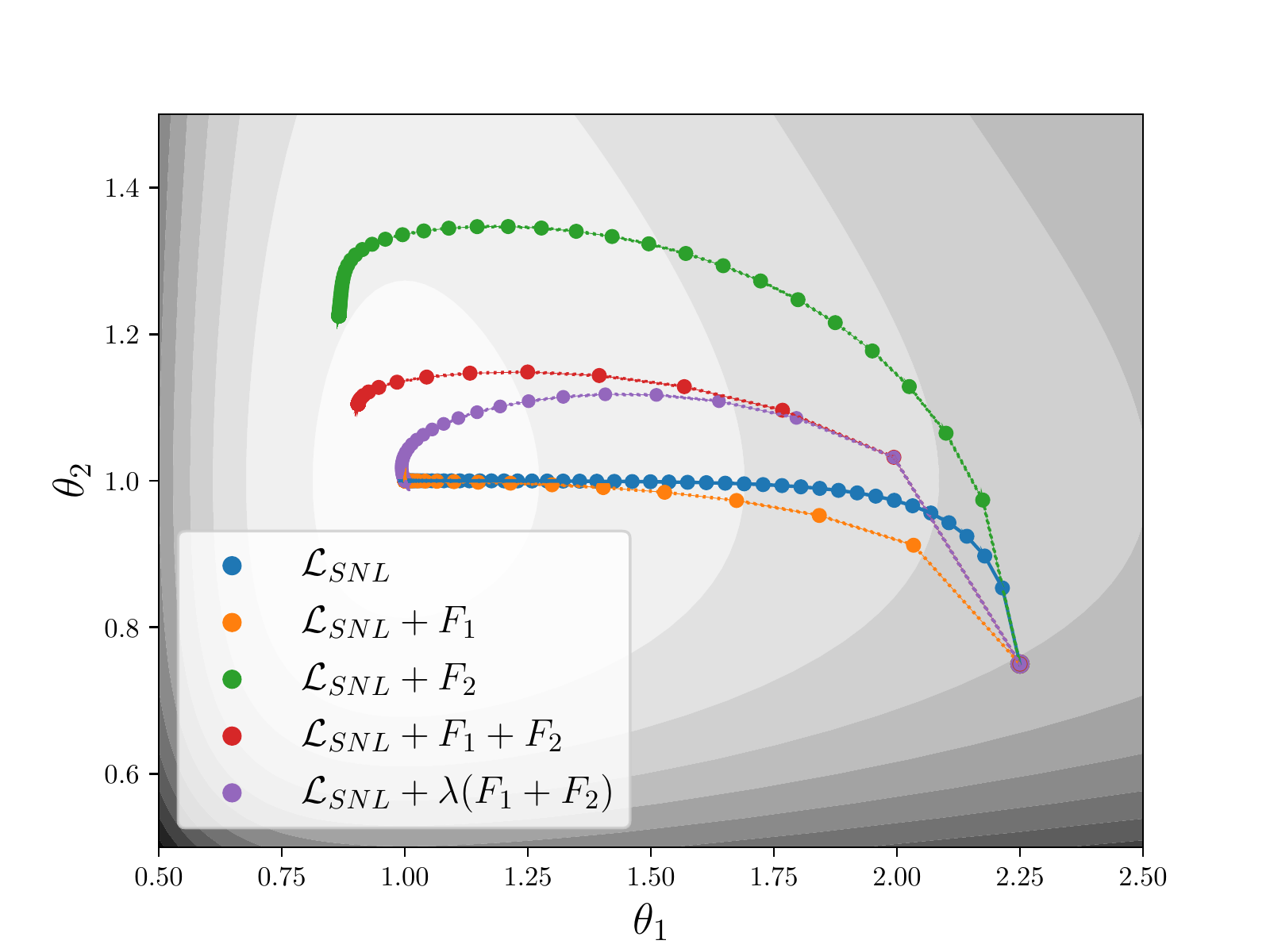}
    \vskip -0.05in
    \caption{Optimization trajectories for various loss candidates in a toy example. Either of $\mathcal{L}_{SNL}$ and $\mathcal{L}_{SNL}+F_{1}$ converges to the optimum, where $F_{1}$ (reverse KL) encourages the faster convergence. The regularization with $F_{2}$ blocks the learning curve to converge to the global optimum, but the curve of $\mathcal{L}_{SNL}+\lambda(F_{1}+F_{2})$ converges to the global optimum if we anneal $\lambda$ exponentially decaying.}
    \label{fig:contour_PRL}
    \vskip -0.05in
\end{figure}

On the other hand, we propose to utilize the mutual information in place of $F_{2}$ constraint by
\begin{align*}
F_{2}(\mathbf{\psi})&=-\mathbb{I}(\bm{\theta},\mathbf{x})=-D_{KL}\Big(\tilde{p}_{r}(\bm{\theta})q_{\mathbf{\psi}}(\mathbf{x}\vert\bm{\theta})\Big\Vert\tilde{p}_{r}(\bm{\theta})\tilde{q}_{\mathbf{\psi}}(\mathbf{x})\Big),
\end{align*}
where $\tilde{q}_{\psi}(\mathbf{x}):=\mathbb{E}_{\tilde{p}_{r}(\bm{\theta})}[q_{\psi}(\mathbf{x}\vert\bm{\theta})]$ is the expected neural likelihood. The maximum mutual information principle enforces the coupling of $\bm{\theta}$ and $\mathbf{x}$ with the neural likelihood $q_{\psi}(\mathbf{x}\vert\bm{\theta})$ to be informative. This maximum principle is particularly effective in resolving the mode collapse problem \citep{lee2020infomax} as well as capturing the rich representation on high dimensions \citep{bachman2019learning}.

We investigate the effect of our regularization in a toy case with a tractable likelihood of $p(x_{1},x_{2})=\mathcal{N}(x_{1};0,1)\mathcal{N}(x_{2};x_{1},1)$. If the model likelihood is $q(x_{1},x_{2})=\mathcal{N}(x_{1},0,\theta_{1}^{2})\mathcal{N}(x_{2};x_{1},\theta_{2}^{2})$, we can derive closed-forms of $\mathcal{L}_{SNL}, F_1$, and $F_2$:
\begin{align*}
\mathcal{L}_{SNL}&=\log{\theta_{1}}+\log{\theta_{2}}+\frac{\theta_{1}^{2}+\theta_{2}^{2}}{2\theta_{1}^{2}\theta_{2}^{2}},\\
F_{1}&=-\log{\theta_{1}}-\log{\theta_{2}}+\frac{\theta_{1}^{2}+\theta_{2}^{2}}{2}-1,\\
F_{2}&=\frac{1}{2}\log{\theta_{1}^{2}+\theta_{2}^{2}}-\log{\theta_{2}}.
\end{align*}
In this toy example, we optimize $\theta_{1}$ and $\theta_{2}$ with respect to the given losses, and mutual information seems to be redundant for the optimization in Fig. \ref{fig:contour_PRL}. However, it turns out that mutual information is key to the better convergence when $F_{1}$ and $F_{2}$ are intractable.

We show how $F_{2}$ affects the inference on SLCP-16 \cite{kim2020sequential} with intractable $F_{1}$ and $F_{2}$. As $F_{1}$ is intractable to calculate, we use adversarial training \cite{nowozin2016f} for the scalable estimation at the cost of training instability. Combined with the mutual information, however, we could avoid using adversarial training to estimate $F(\psi)$, and Fig. \ref{fig:SLCP-16_sample} shows that SNL with our regularization outperforms SNL regularized by $F_{1}$. Quantitatively, we measure the relative mutual information by 
\begin{align*}
\mathbb{I}_{rel}(r):=\frac{D_{KL}\big(\tilde{p}_{r}(\bm{\theta})q_{\mathbf{\psi}}(\mathbf{x}\vert\bm{\theta})\Vert\tilde{p}_{r}(\bm{\theta})\tilde{q}_{\mathbf{\psi}}(\mathbf{x})\big)}{D_{KL}\big(\tilde{p}_{r}(\bm{\theta})p_{sim}(\mathbf{x}\vert\bm{\theta})\Vert\tilde{p}_{r}(\bm{\theta})\tilde{p}_{sim}(\mathbf{x})\big)},
\end{align*}
where $\tilde{p}_{sim}(\mathbf{x})=\mathbb{E}_{\tilde{p}_{r}(\bm{\theta})}[p_{sim}(\mathbf{x}\vert\bm{\theta})]$. This relative mutual information satisfies $\mathbb{I}_{rel}(r)=1$ if and only if the neural likelihood $q_{\psi}(\mathbf{x}\vert\bm{\theta})$ exactly matches the ground-truth likelihood $p_{sim}(\mathbf{x}\vert\bm{\theta})$. In Fig. \ref{fig:MI}, we compare three loss candidates that converge to the global optimum of Fig. \ref{fig:contour_PRL}: $\mathcal{L}_{SNL}$, $\mathcal{L}_{SNL}+F_{1}$, and $\mathcal{L}_{SNL}+\lambda(F_{1}+F_{2})$. Fig. \ref{fig:MI} shows that $\mathcal{L}_{SNL}$ regularized by $F_{1}+F_{2}$ with proper choice of regularization coefficient $\lambda$ is most close to the dotted line of $q_{\psi}(\mathbf{x}\vert\bm{\theta})=p_{sim}(\mathbf{x}\vert\bm{\theta})$ among all. 

\section{Neural Posterior Regularization}\label{sec:Methodology}

\subsection{Unified Estimation of Reverse KL and Mutual Information}

While $F=F_{1}+F_{2}$ is designed to avoid the inefficient feedback loop mentioned in Section \ref{sec:SNL_problem}, the constraint $F$ cannot be tractably computed in general. Therefore, we introduce a method to approximate the constraint. To begin with, recall that the neural posterior loss is
\begin{align*}
\mathcal{L}_{APT}(\mathbf{\phi})=-\mathbb{E}_{\tilde{p}_{r}(\bm{\theta})p_{sim}(\mathbf{x}\vert\bm{\theta})}\bigg[\log{\frac{q_{\mathbf{\phi}}(\bm{\theta}\vert\mathbf{x})}{Z_{\mathbf{\phi}}(\mathbf{x})}}\bigg].
\end{align*}
By replacing the expected distribution of $\mathcal{L}_{APT}$ from $\tilde{p}_{r}(\bm{\theta})p_{sim}(\mathbf{x}\vert\bm{\theta})$ to $\tilde{p}_{r}(\bm{\theta})q_{\psi}(\mathbf{x}\vert\bm{\theta})$, the loss $\mathcal{L}_{APT}$ transforms to
\begin{align}\label{eq:NPR}
F(\mathbf{\psi},\mathbf{\phi})=-\mathbb{E}_{\tilde{p}_{r}(\bm{\theta})q_{\mathbf{\psi}}(\mathbf{x}\vert\bm{\theta})}\bigg[\log{\frac{q_{\mathbf{\phi}}(\bm{\theta}\vert\mathbf{x})}{Z_{\mathbf{\phi}}(\mathbf{x})}}\bigg].
\end{align}
When the neural posterior satisfies $q_{\mathbf{\phi}^{*}}(\bm{\theta}\vert\mathbf{x})=p_{sim}(\bm{\theta}\vert\mathbf{x})$ and $Z_{\mathbf{\phi}^{*}}(\mathbf{x})=\frac{\tilde{p}_{r}(\mathbf{x})}{p_{sim}(\mathbf{x})}$ for $p_{sim}(\mathbf{x})=\mathbb{E}_{p(\bm{\theta})}[p_{sim}(\mathbf{x}\vert\bm{\theta})]$, Eq. \ref{eq:NPR} reduces to
\begin{align*}
F(\mathbf{\psi},\mathbf{\phi}^{*})=-\mathbb{E}_{\tilde{p}_{r}(\bm{\theta})q_{\mathbf{\psi}}(\mathbf{x}\vert\bm{\theta})}\bigg[\log{\frac{p_{sim}(\bm{\theta}\vert\mathbf{x})p_{sim}(\mathbf{x})}{\tilde{p}_{r}(\mathbf{x})}}\bigg].
\end{align*}
Then, Theorem \ref{thm:1} proves that $F(\mathbf{\psi},\mathbf{\phi}^{*})$ is the proxy of the regularization of $F(\mathbf{\psi})$.
\begin{theorem}\label{thm:1}
    $F(\mathbf{\psi},\mathbf{\phi}^{*})$ is decomposed into $F(\mathbf{\psi},\mathbf{\phi}^{*})=F(\mathbf{\psi})-R(\mathbf{\psi})$, up to a constant, where $R(\mathbf{\psi})$ is the residual term given by $R(\mathbf{\psi})=D_{KL}\big(\tilde{q}_{\mathbf{\psi}}(\mathbf{x})\Vert\tilde{p}_{r}(\mathbf{x})\big)$.
\end{theorem}
\begin{proof}
	We have
	\begin{eqnarray*}
	\lefteqn{F(\psi,\phi^{*})=-\mathbb{E}_{\tilde{p}_{r}(\bm{\theta})q_{\psi}(\mathbf{x}\vert\bm{\theta})}\bigg[\log{\frac{p_{sim}(\bm{\theta}\vert\mathbf{x})p_{sim}(\mathbf{x})}{\tilde{p}_{r}(\mathbf{x})}}\bigg]}&\\
	&&=\int \int \tilde{p}_{r}(\bm{\theta})q_{\psi}(\mathbf{x}\vert\bm{\theta})\log{\frac{p(\bm{\theta})q_{\psi}(\mathbf{x}\vert\bm{\theta})}{p_{sim}(\bm{\theta}\vert\mathbf{x})p_{sim}(\mathbf{x})}}\diff\bm{\theta}\diff\mathbf{x}\\
	&&\quad-\int\int \tilde{p}_{r}(\bm{\theta})q_{\psi}(\mathbf{x}\vert\bm{\theta})\log{\frac{\tilde{p}_{r}(\bm{\theta})q_{\psi}(\mathbf{x}\vert\bm{\theta})}{\tilde{p}_{r}(\bm{\theta})\tilde{q}_{\psi}(\mathbf{x})}}\diff\bm{\theta}\diff\mathbf{x}\\
	&&\quad-\int\int \tilde{p}_{r}(\bm{\theta})q_{\psi}(\mathbf{x}\vert\bm{\theta})\log{\frac{\tilde{q}_{\psi}(\mathbf{x})}{\tilde{p}_{r}(\mathbf{x})}}\diff\bm{\theta}\diff\mathbf{x}\\
	&&\quad-\int\int \tilde{p}_{r}(\bm{\theta})q_{\psi}(\mathbf{x}\vert\bm{\theta})\log{p(\bm{\theta})}\diff\bm{\theta}\diff\mathbf{x}\\
	&&=D_{KL}\big(\tilde{p}_{r}(\bm{\theta})q_{\psi}(\mathbf{x}\vert\bm{\theta})\Vert \tilde{p}_{r}(\bm{\theta})p_{sim}(\mathbf{x}\vert\bm{\theta})\big)\\
	&&\quad-D_{KL}\Big(\tilde{p}_{r}(\bm{\theta})q_{\mathbf{\psi}}(\mathbf{x}\vert\bm{\theta})\Big\Vert\tilde{p}_{r}(\bm{\theta})\tilde{q}_{\mathbf{\psi}}(\mathbf{x})\Big)\\
	&&\quad-D_{KL}\big(\tilde{q}_{\psi}(\mathbf{x})\Vert\tilde{p}_{r}(\mathbf{x})\big)+\int\tilde{p}_{sim}(\bm{\theta})\log{p(\bm{\theta})}\diff\bm{\theta}\\
	&&=F(\psi)-R(\psi)+C,
	\end{eqnarray*}
	where $C$ is irrelevant to $\psi$.
\end{proof}

\begin{figure}[t]
\vskip -0.05in
\centering
    \includegraphics[width=0.75\linewidth]{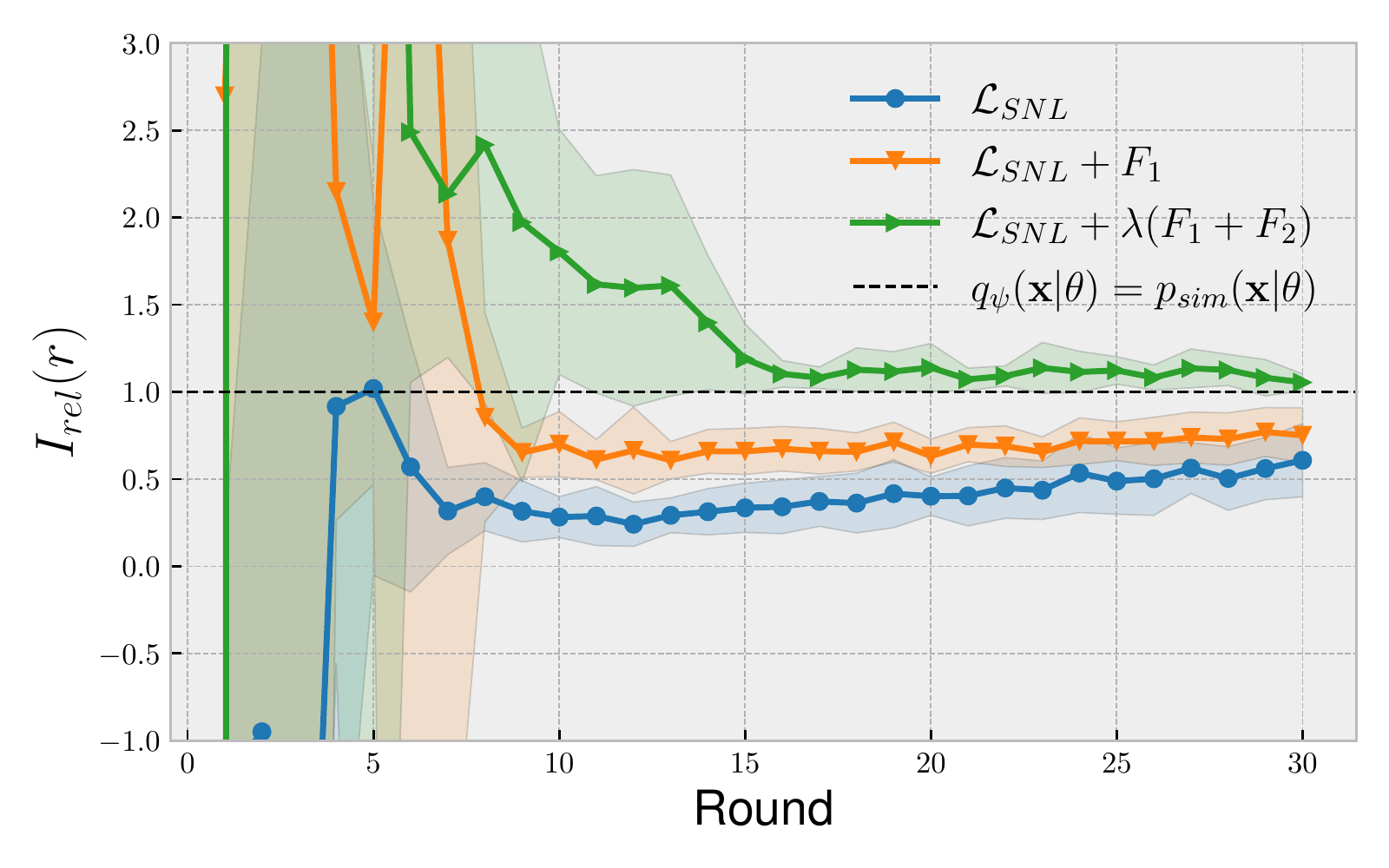}
    \vskip -0.05in
    \caption{Relative mutual information by inference rounds.}
    \label{fig:MI}
    \vskip -0.05in
\end{figure}

\begin{table*}[t]
    \caption{Comparison of performance on SLCP-16 with 100-dimensional and SLCP-256 with 80-dimensional output. The NLTP of SLCP-256 is scaled by $10^{-2}$.}
    \label{tab:performance}
    \scriptsize
    \centering
    \resizebox{.72\linewidth}{!}{
        \begin{tabular}{crrrrrr}
            \toprule
            \multirow{2}{*}{Algorithm} & \multicolumn{3}{c}{SLCP-16} & \multicolumn{3}{c}{SLCP-256}\\
            \cmidrule(lr){2-4}
            \cmidrule(lr){5-7}
            & \multicolumn{1}{c}{NLTP ($\downarrow$)} & \multicolumn{1}{c}{$\log{\text{MMD}}$ ($\downarrow$)} & \multicolumn{1}{c}{IS ($\uparrow$)} & \multicolumn{1}{c}{NLTP ($\downarrow$)} & \multicolumn{1}{c}{$\log{\text{MMD ($\downarrow$)}}$} & \multicolumn{1}{c}{IS ($\downarrow$)} \\\midrule
            SMC-ABC & 8.21$\pm$2.36 & -1.80$\pm$1.24 & 13.56$\pm$0.55 & 75.44$\pm$29.71 & -4.48$\pm$0.55 & 81.69$\pm$7.73\\
            SNPE-A & 1.46$\pm$1.72 & -4.47$\pm$0.05 & 2.66$\pm$0.23 & 39.08$\pm$1.72 & -4.55$\pm$0.21 & 34.74$\pm$3.68\\
            SNPE-B & 27.66$\pm$23.83 & -2.12$\pm$0.31 & 1.50$\pm$0.40 & 140.4$\pm$22.78 & -1.83$\pm$0.13 & 19.43$\pm$0.93\\
            APT (SNPE-C) & 3.14$\pm$8.80 & -3.27$\pm$0.71 & 7.39$\pm$4.48 & 69.04$\pm$93.53 & -3.51$\pm$1.24 & 118.88$\pm$51.30\\
            AALR & 0.89$\pm$0.12 & -3.57$\pm$0.43 & 11.22$\pm$3.37 & 16.19$\pm$1.04 & \textbf{-6.81}$\pm$0.68 & 208.91$\pm$4.33\\
            SNL & 4.77$\pm$2.68 & -2.53$\pm$0.54 & 5.34$\pm$3.43 & 40.40$\pm$11.84 & -5.32$\pm$0.65 & 153.43$\pm$18.40\\\midrule
            SNL+NPR & \textbf{0.55}$\pm$0.79 & \textbf{-5.39}$\pm$0.94 & \textbf{14.95}$\pm$1.08 & \textbf{15.13}$\pm$2.65 & -6.51$\pm$0.86 & \textbf{211.85}$\pm$4.07\\
            \bottomrule
        \end{tabular}
    }
\end{table*}

The parameter $\mathbf{\phi}$ estimates the regularization by training the neural posterior, and $\mathbf{\psi}$ estimates the ground-truth likelihood by training the neural likelihood. We call $F(\mathbf{\psi},\mathbf{\phi})$ the Neural Posterior Regularization (NPR) since $F(\psi,\phi)$ is computed based on the neural posterior evaluation. Hence, we highlight that our regularization is indeed a unified framework of SNL and APT, and it enjoys the benefits of APT and SNL.

Altogether, we introduce the regularized loss function as
\begin{align}\label{eq:regularized_loss}
\mathcal{L}_{\lambda}(\mathbf{\psi},\mathbf{\phi})=\mathcal{L}_{SNL}(\mathbf{\psi})+\lambda F(\mathbf{\psi}, \mathbf{\phi}).
\end{align}
This regularized loss approximates the constrained problem of Eq. \ref{eq:constrained_problem} at the expense of additional neural parameter usage ($\phi$). However, it is worth noting that the main interest of \textit{likelihood-free inference} is minimizing the simulation budget rather than reducing the number of neural parameters.

\begin{table*}[t]
    \caption{Comparison of NLTP for the M/G/1 model, the Ricker model, and the Poisson model. Though NLTP on these simulations does not take the structure of the approximate posterior into account, we report this table to comply the previous researches \citep{papamakarios2019sequential}.}
    \vskip -0.1in
    \label{tab:performance_real}
    \scriptsize
    \begin{center}
        \resizebox{.95\linewidth}{!}{
            \begin{tabular}{crrrrrrrrr}
                \toprule
                \multirow{3}{*}{Algorithm} & \multicolumn{3}{c}{M/G/1} & \multicolumn{3}{c}{Ricker} & \multicolumn{3}{c}{Poisson}\\
                \cmidrule(lr){2-4}
                \cmidrule(lr){5-7}
                \cmidrule(lr){8-10}
                &\multicolumn{3}{c}{Output Dimension}&\multicolumn{3}{c}{Output Dimension} & \multicolumn{3}{c}{Output Dimension}\\
                & \multicolumn{1}{c}{5} & \multicolumn{1}{c}{20} & \multicolumn{1}{c}{100} & \multicolumn{1}{c}{13} & \multicolumn{1}{c}{20} & \multicolumn{1}{c}{100} & \multicolumn{1}{c}{25} & \multicolumn{1}{c}{49} & \multicolumn{1}{c}{361}\\
                \midrule
                SMC-ABC & 22.54$\pm$17.83 & 25.48$\pm$22.12 & 14.85$\pm$18.41 & 6.77$\pm$7.53 & 6.02$\pm$7.61 & 36.12$\pm$18.64 & 0.92$\pm$1.42 & 0.58$\pm$1.75 & 0.00$\pm$0.00\\
                SNPE-A & 4.36$\pm$0.43 & 4.07$\pm$0.49 & 3.25$\pm$0.61 & 4.43$\pm$0.20 & 4.54$\pm$0.49 & 4.43$\pm$0.18 & 0.34$\pm$0.48 & 0.17$\pm$0.42 & -3.22$\pm$6.32\\
                SNPE-B & 9.30$\pm$0.96 & 9.85$\pm$0.03 & 9.89$\pm$0.01 & 9.79$\pm$0.06 & 8.67$\pm$1.16 & 9.96$\pm$0.25 & 5.44$\pm$5.59 & 5.59$\pm$4.92 & 3.47$\pm$4.67\\
                APT (SNPE-C) & -3.49$\pm$0.97 & -2.27$\pm$7.14 & -3.01$\pm$1.82 & -0.62$\pm$0.98 & -0.24$\pm$3.86 & 4.13$\pm$6.16 & -11.61$\pm$1.93 & -11.61$\pm$2.50 & \textbf{-11.26}$\pm$0.96 \\
                AALR & -1.18$\pm$0.59 & -1.69$\pm$0.51 & 1.14$\pm$2.08 & 4.10$\pm$0.51 & 1.82$\pm$0.77 & 3.15$\pm$2.55 & -6.36$\pm$1.66 & -6.25$\pm$0.92 & 0.28$\pm$0.66 \\                             
                SNL & -3.23$\pm$1.25 & -5.58$\pm$1.37 & -4.65$\pm$1.30 & -1.46$\pm$0.66 & -0.41$\pm$1.07 & 1.44$\pm$6.42 & -12.76$\pm$0.99 & -12.48$\pm$0.95 & -2.36$\pm$7.44 \\\midrule
                SNL+NPR & \textbf{-3.86}$\pm$0.68 & \textbf{-6.36}$\pm$0.78 & \textbf{-5.52}$\pm$0.83 & \textbf{-1.78}$\pm$0.89 & \textbf{-1.14}$\pm$1.25 & \textbf{-0.48}$\pm$0.70 & \textbf{-13.31}$\pm$1.11 & \textbf{-13.61}$\pm$0.83 & -5.33$\pm$4.33 \\
                \bottomrule
            \end{tabular}
        }
    \end{center}
    \vskip -0.1in
\end{table*}

\subsection{Optimality Analysis of NPR}

The optimal neural likelihood $q_{\psi_{\lambda}^{*}}(\mathbf{x}\vert\bm{\theta})$ of $\mathcal{L}_{\lambda}(\psi,\phi)$ could deviate too far from $p_{sim}(\mathbf{x}\vert\bm{\theta})$. Therefore, we analyze the optimal point of the regularized loss $\mathcal{L}_{\lambda}(\psi,\phi)$ in Theorem \ref{thm:2}.
\begin{theorem}\label{thm:2}
    Suppose $p_{sim}(\mathbf{x}\vert\bm{\theta})$ and $p(\bm{\theta})$ are bounded on a positive interval, $\tilde{p}_{r}(\bm{\theta})$ is bounded below, and $q_{\mathbf{\psi}}(\mathbf{x}\vert\bm{\theta})$ is uniformly upper bounded on $\mathbf{\psi}$. Then $\psi_{\lambda}^{*}=\argmin_{\psi}\mathcal{L}_{\lambda}(\mathbf{\psi},\mathbf{\phi})$ satisfies
    \begin{align*}
    q_{\psi_{\lambda}^{*}}(\mathbf{x}\vert\bm{\theta})=\frac{p_{sim}(\mathbf{x}\vert\bm{\theta})}{c(\bm{\theta})-\lambda\log{\frac{q_{\mathbf{\phi}}(\bm{\theta}\vert\mathbf{x})}{Z_{\mathbf{\phi}}(\mathbf{x})}}},
    \end{align*}
    where $c(\bm{\theta})$ is a function of $\bm{\theta}$ that makes $q_{\mathbf{\psi}^{*}}^{\lambda}(\mathbf{x}\vert\bm{\theta})$ a distribution.
\end{theorem}
\begin{proof}
Suppose $u(\bm{\theta},\mathbf{x})$ is a function of $(\bm{\theta},\mathbf{x})$ that satisfies $\int u(\bm{\theta},\mathbf{x})dx=0$ for any $\bm{\theta}$, then the function $(q+\epsilon r)(\mathbf{x}\vert\bm{\theta}):= q(\mathbf{x}\vert\bm{\theta})+\epsilon u(\bm{\theta},\mathbf{x})$ is a distribution on $\mathbf{x}$ for any $\bm{\theta}$. With the abuse of notation, the difference $\mathcal{L}_{\lambda}(q_{\psi}+\epsilon u,q_{\phi})-\mathcal{L}_{\lambda}(q_{\psi})$ becomes
	\begin{align*}
		&\mathcal{L}_{\lambda}(q_{\psi}+\epsilon u,q_{\phi})-\mathcal{L}_{\lambda}(q_{\psi},q_{\phi})\numberthis\label{eq:frechet}\\
		&=-\epsilon\int\int\tilde{p}_{r}(\bm{\theta})u(\bm{\theta},\mathbf{x})\bigg(\frac{p(\mathbf{x}\vert\bm{\theta})}{q_{\psi}(\mathbf{x}\vert\bm{\theta})}+\lambda\log{\frac{q_{\phi}(\bm{\theta}\vert\mathbf{x})}{Z_{\phi}(\mathbf{x})}}\bigg)\diff\bm{\theta}\diff\mathbf{x}+o(\epsilon).
	\end{align*}
	Therefore, the optimal solution of $\argmin_{\psi}\mathcal{L}_{\lambda}(\psi,\phi)$ is where Eq. \ref{eq:frechet} becomes zero for all $u(\bm{\theta},\mathbf{x})$ with $\int u(\bm{\theta},\mathbf{x})\diff\mathbf{x}=0$, $\forall\bm{\theta}\in\bm{\theta}$. By the canonical calculus using the Minkowski and H\"{o}lder's inequalities, we derive that
	\begin{align*}
	\frac{p_{sim}(\mathbf{x}\vert\bm{\theta})}{q_{\psi_{\lambda}^{*}}(\mathbf{x}\vert\bm{\theta})}+\lambda\log{\frac{q_{\phi}(\bm{\theta}\vert\mathbf{x})}{Z_{\phi}(\mathbf{x})}}=c(\bm{\theta}),
	\end{align*}
	from Eq. \ref{eq:frechet} for some function $c(\bm{\theta})$ that makes $q_{\psi_{\lambda}^{*}}(\mathbf{x}\vert\bm{\theta})$ a distribution. Also, the canonical analysis proves the uniqueness of the optimal distribution, which completes the proof.
	\end{proof}

\section{Study on Regularization Coefficient}\label{sec:regularization_coefficient}

Searching for the optimal $\lambda$ would be highly impractical if the simulation budget is strictly limited. Hence, the optimal strategy of $\lambda$ is required a-priori for \textit{likelihood-free inference}. Fig. \ref{fig:hyperparameter} illustrates the inference quality by round on SLCP-16/256 \citep{kim2020sequential} with 16/256 modes, respectively. We measure inference quality by the Inception Score (IS) \citep{salimans2016improved}, which counts the diversity of the approximate posterior. A higher IS indicates a better inference.

\begin{figure}[t]
\vskip -0.1in
    \centering
    \subfigure[SLCP-16]{\includegraphics[width=.49\linewidth]{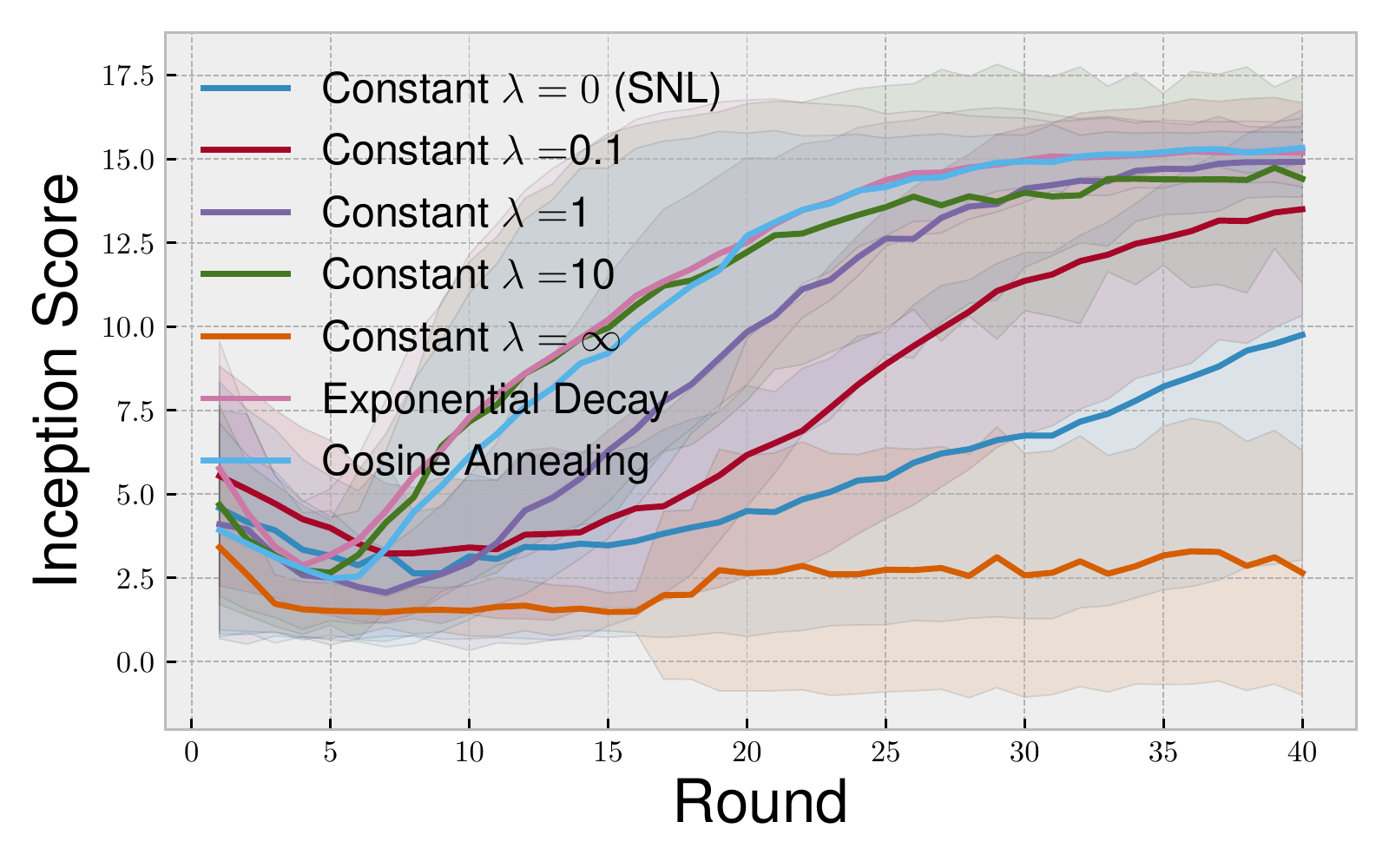}}
    \centering
    \subfigure[SLCP-256]{\includegraphics[width=.49\linewidth]{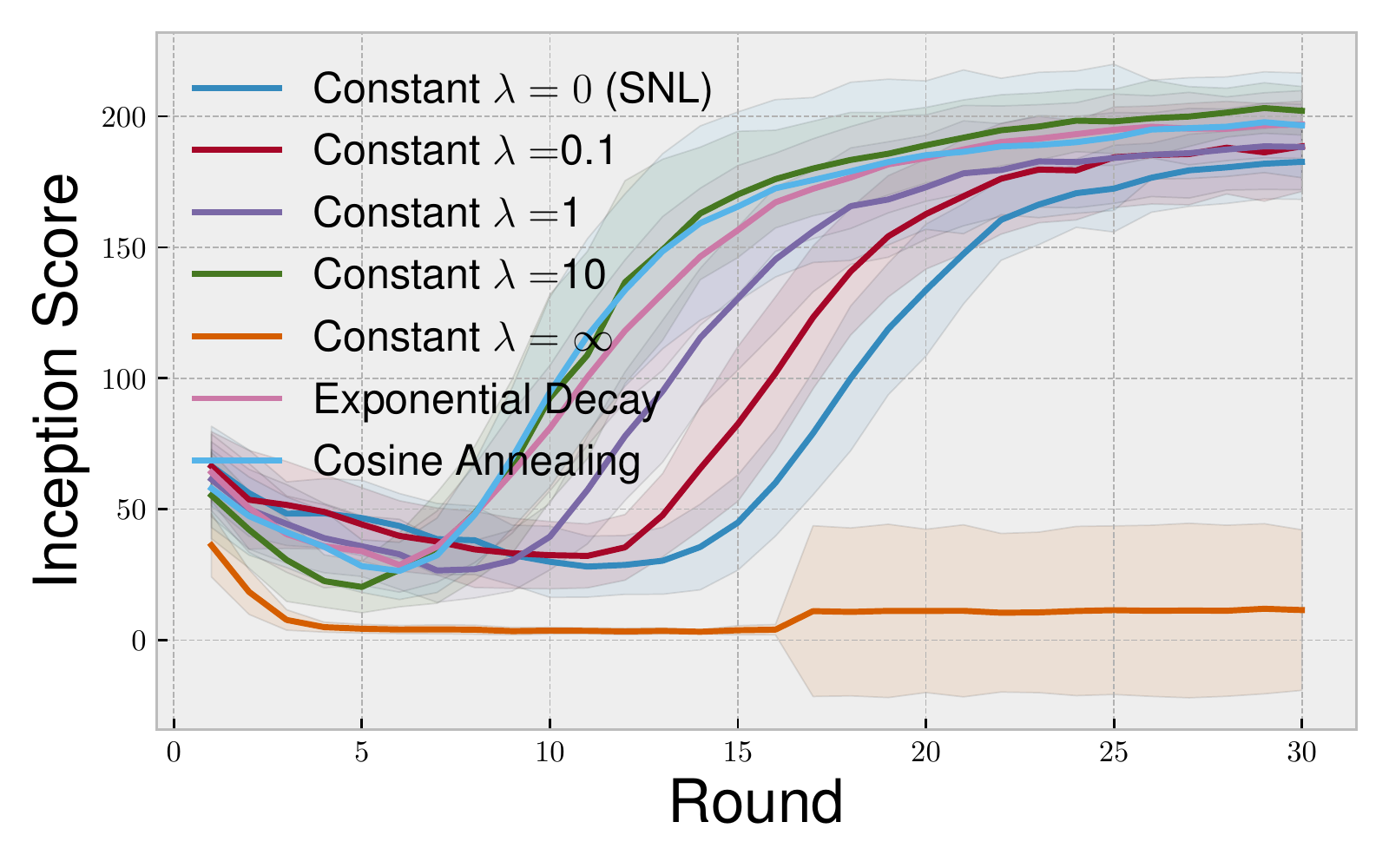}}
    \vskip -0.05in
    \caption{Effect of the regularization coefficient on the inference quality.}
    \label{fig:hyperparameter}
\vskip -0.1in
\end{figure}

\begin{figure*}[t]
\vskip -0.05in
    \centering
    \subfigure[Groundtruth]{\includegraphics[width=.195\linewidth]{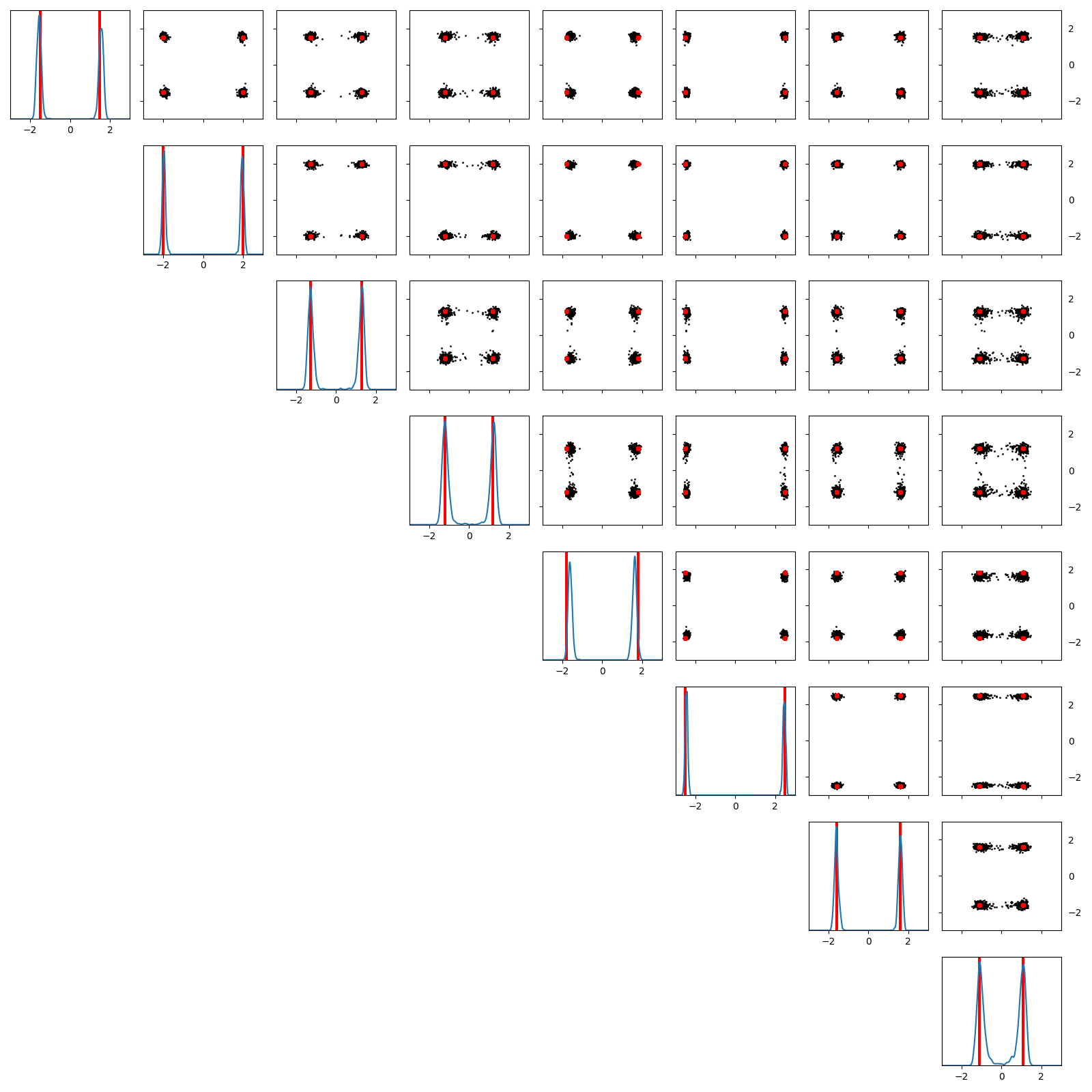}}
    \centering
    \subfigure[APT]{\includegraphics[width=.195\linewidth]{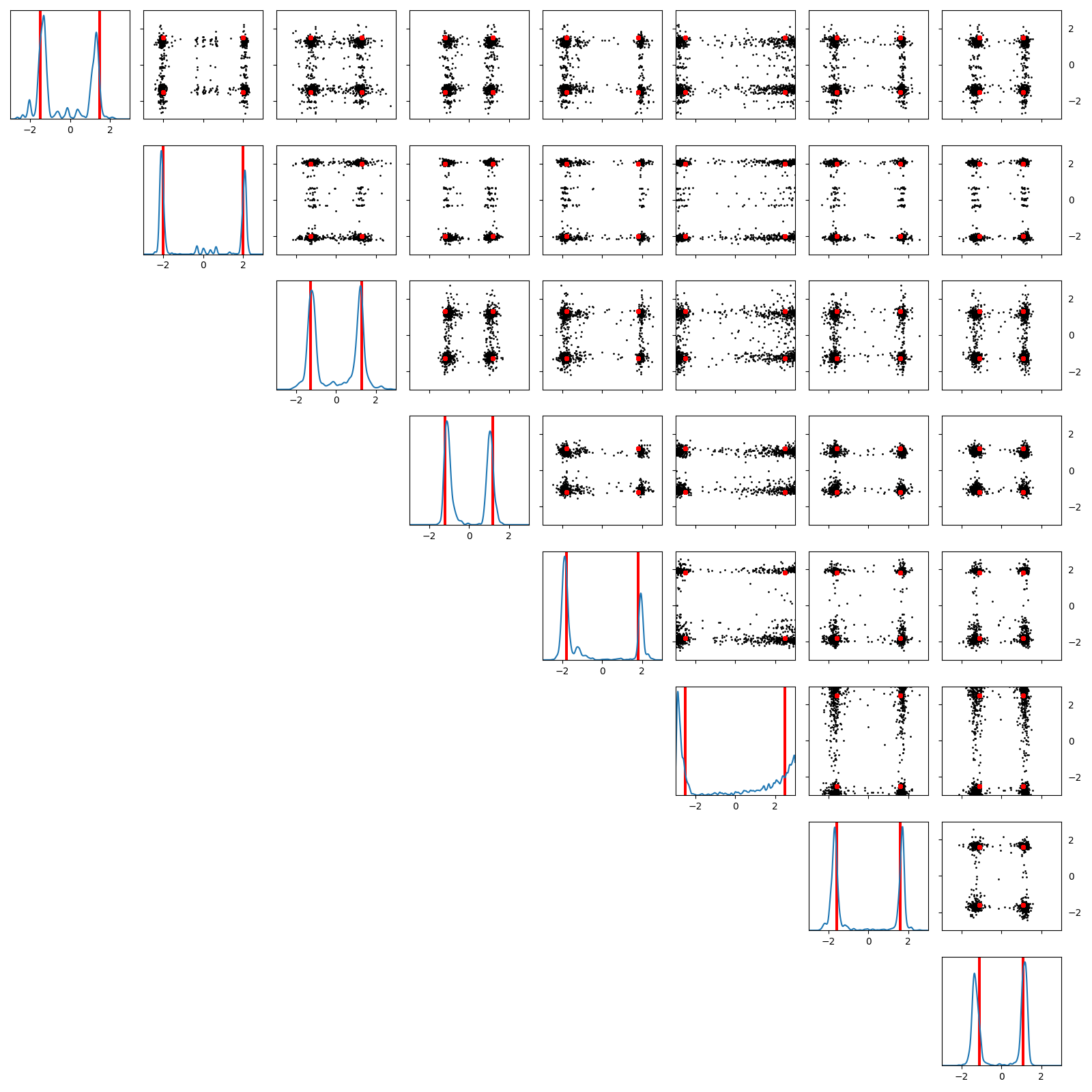}}
    \centering
    \subfigure[AALR]{\includegraphics[width=.195\linewidth]{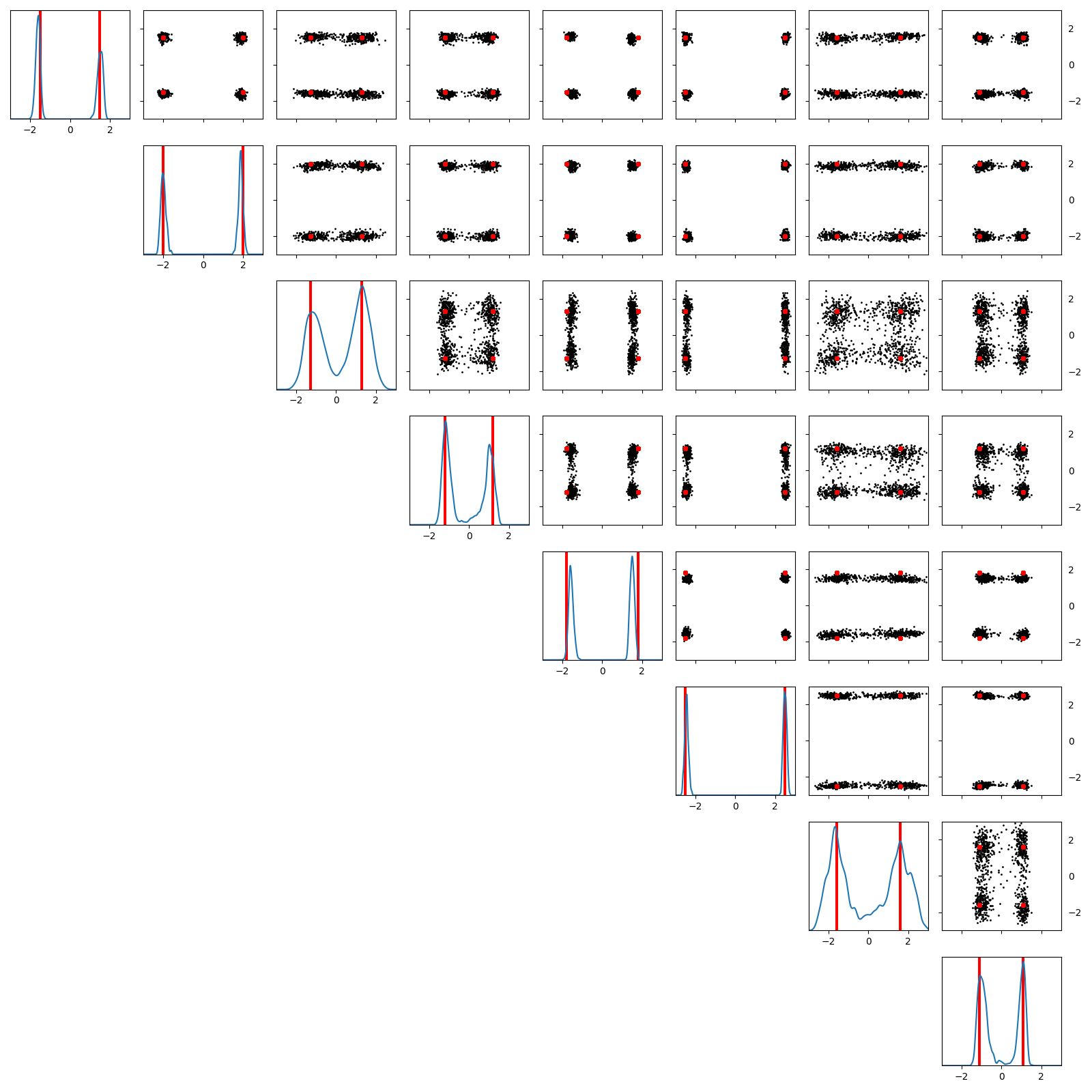}}
    \centering
    \subfigure[SNL]{\includegraphics[width=.195\linewidth]{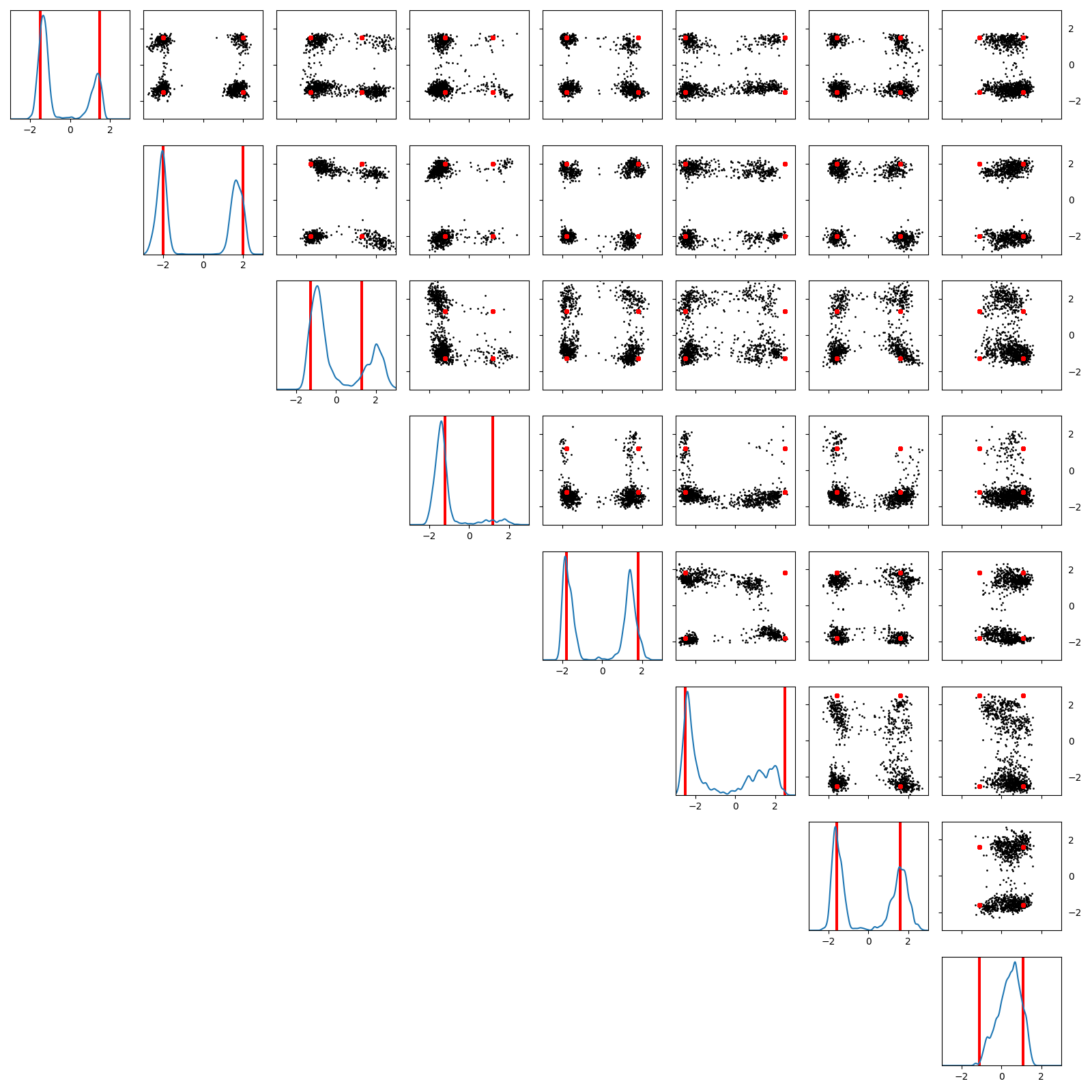}}
    \centering
    \subfigure[SNL+NPR]{\includegraphics[width=.195\linewidth]{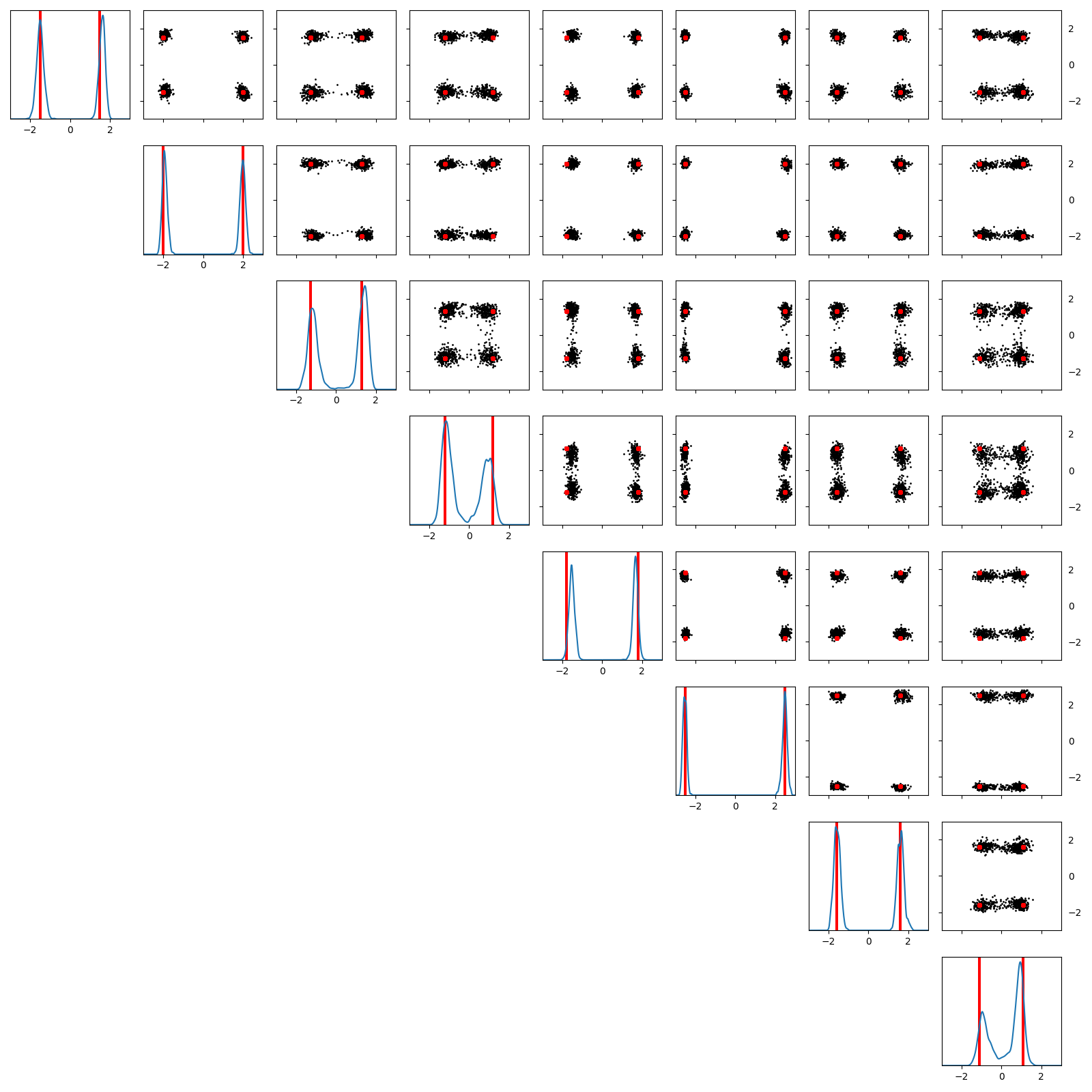}}
    \vskip -0.05in
    \caption{Comparison of the approximate posterior distributions on SLCP-256.}
    \label{fig:marginal_posterior}
    \vskip -0.05in
\end{figure*}

\begin{figure}[t]
\vskip -0.05in
    \centering
    \subfigure[SLCP-16 by Algorithm]{\includegraphics[width=.495\linewidth]{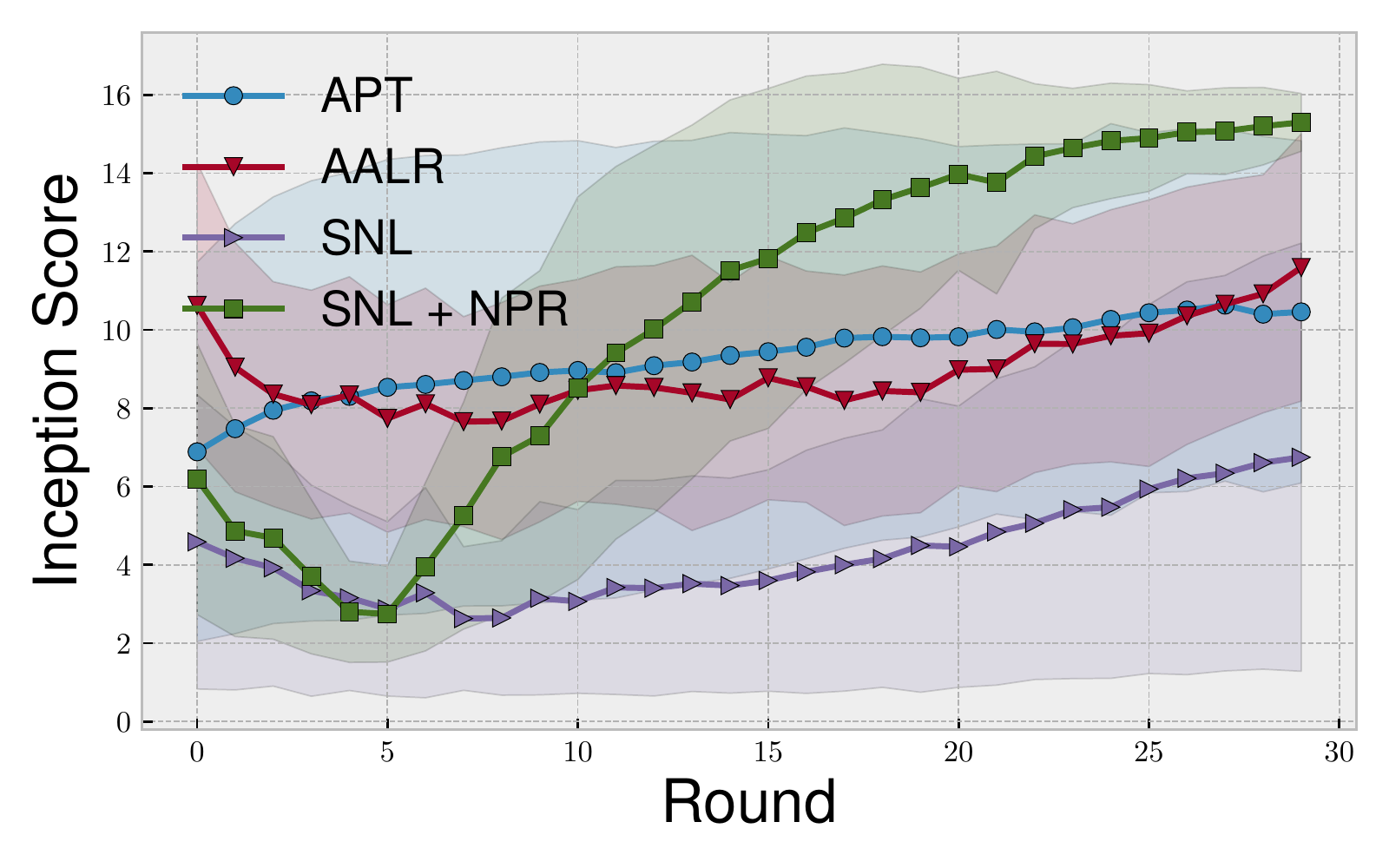}}
    \centering
    \subfigure[SLCP-16 by Dimension]{\includegraphics[width=.495\linewidth]{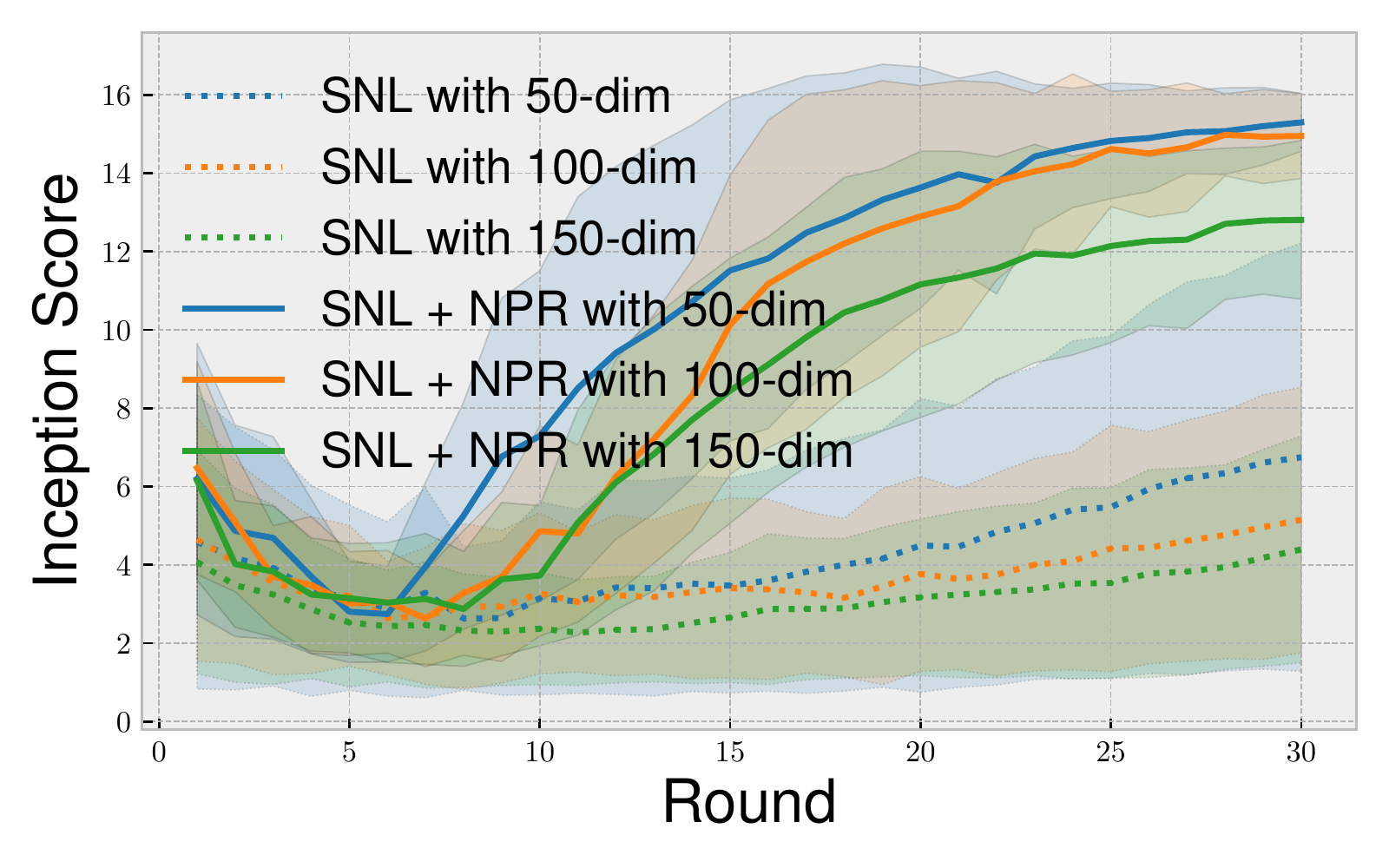}}
\vskip -0.05in
    \caption{IS by round on (a) baselines of 50-dim output and (b) varying dimensions.}
    \label{fig:SLCP-16}
    \vskip -0.05in
\end{figure}

\begin{figure}[t]
\vskip -0.05in
    \centering
    \subfigure[SLCP-256 by Algorithm]{\includegraphics[width=.495\linewidth]{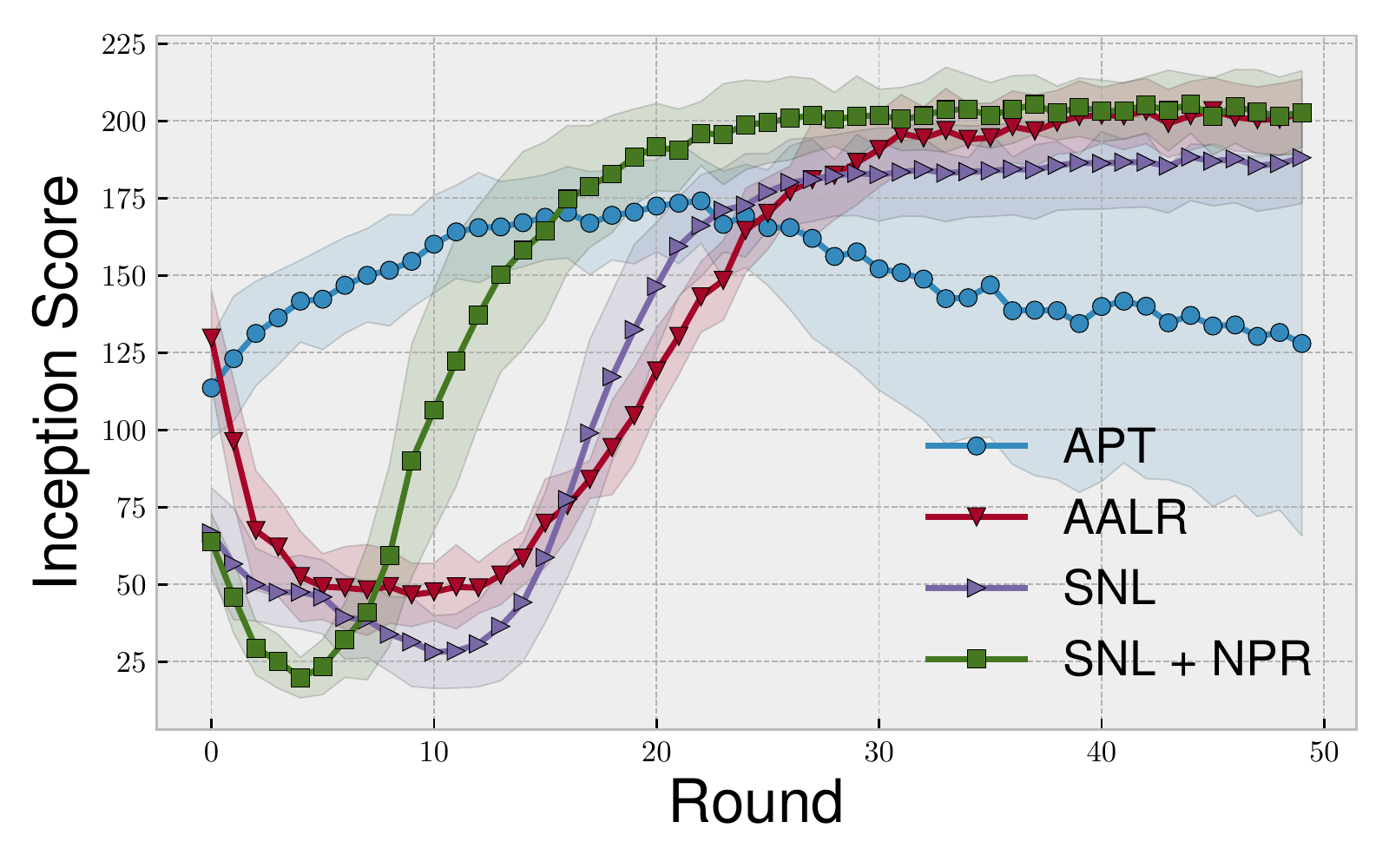}}
    \centering
    \subfigure[SLCP-256 by Dimension]{\includegraphics[width=.495\linewidth]{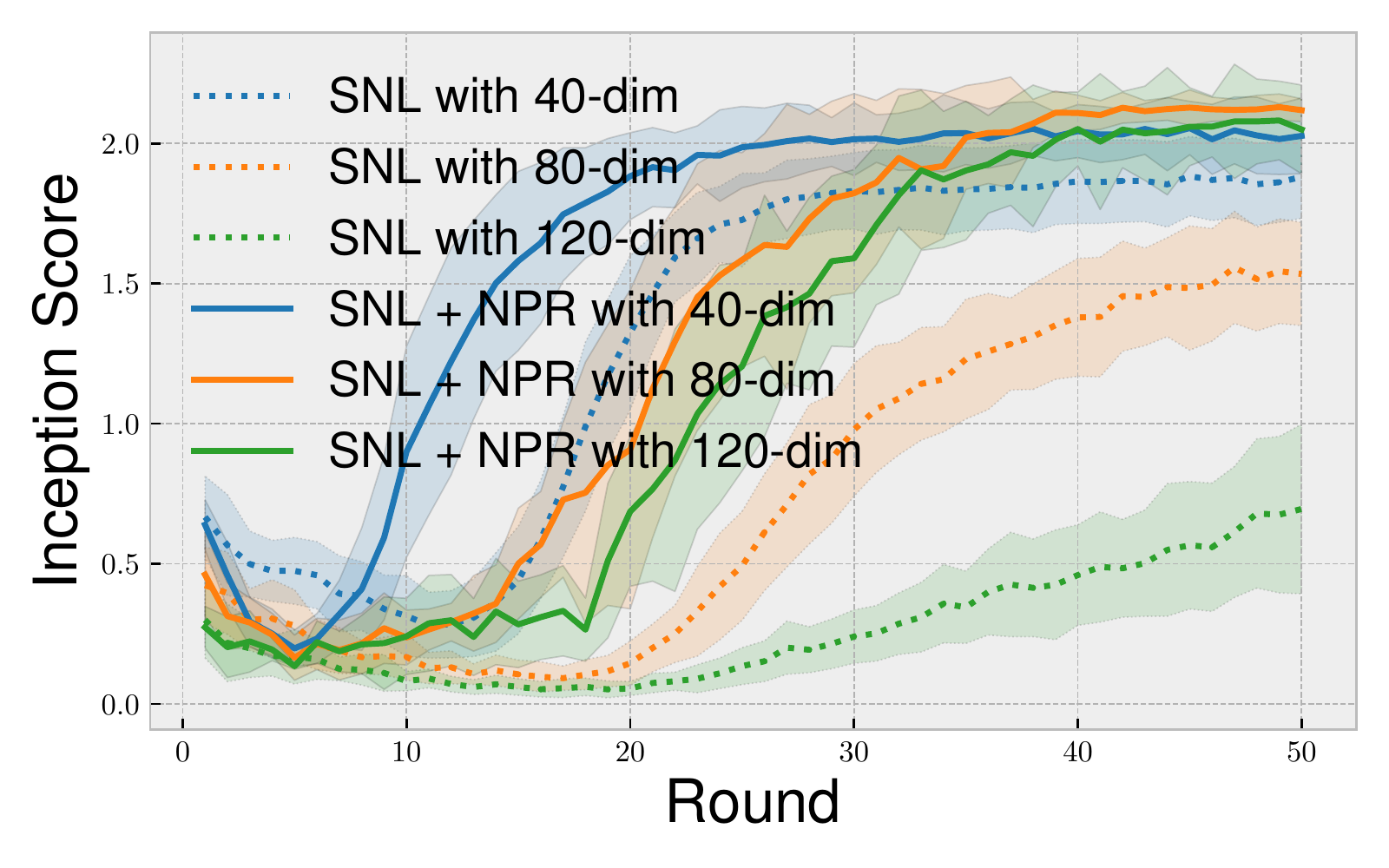}}
    \vskip -0.05in
    \caption{IS by round with (a) baselines of 40-dim output and (b) varying dimensions.}
    \label{fig:SLCP-256}
    \vskip -0.05in
\end{figure}

\begin{figure}[t]
	\centering
		\includegraphics[width=\linewidth]{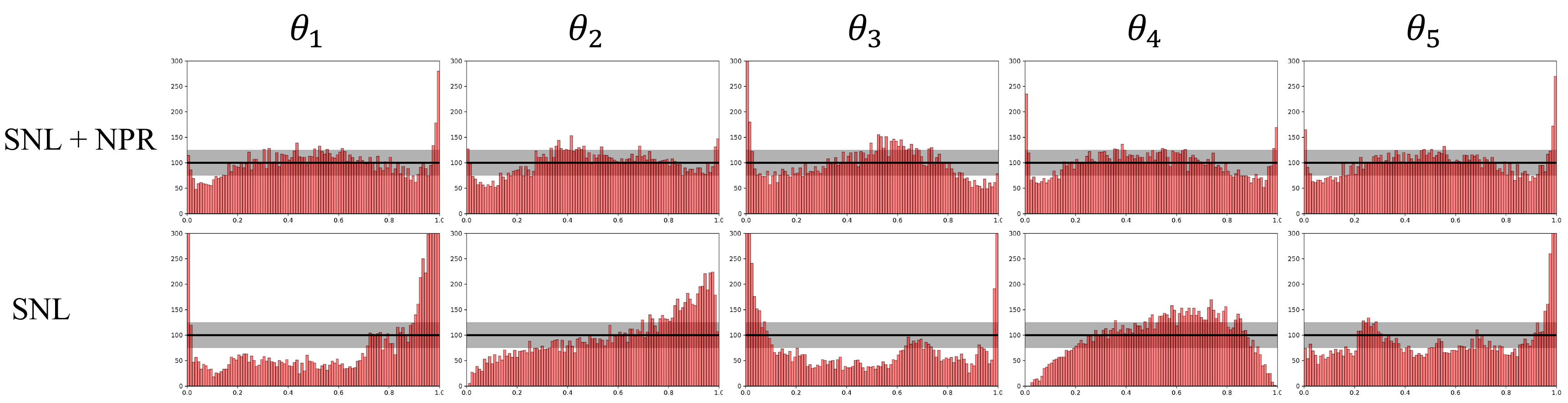}
	\caption{SBC after 10 rounds of inference on the MNIST experiment.}
	\label{fig:MNIST_sbc}
	\vskip -0.05in
\end{figure}

Consistent with the experiment of the toy case in Fig. \ref{fig:contour_PRL}, the multi-round inference in Fig. \ref{fig:hyperparameter} shows that the inference with $\lambda=1$ is faster than that of $\lambda=0$. However, the best $\lambda$ magnitude differs by SLCP-16 and SLCP-256 with $\lambda=1$ and $\lambda=10$, respectively, and this inconsistency by simulation model leads us to propose the $\lambda$-scheduling. Motivated by Theorem \ref{thm:2}, we propose annealing methods for the $\lambda$-scheduling. In Fig. \ref{fig:hyperparameter}, either of the exponential decaying and cosine annealing scheduling methods performs as well as the best $\lambda$ choices, and these adaptive scheduling methods reduce the endeavor of $\lambda$ search. We choose the exponential decaying strategy by default.

\begin{figure}[t]
    \centering
    \subfigure[MEDDIST]{\includegraphics[width=.48\linewidth]{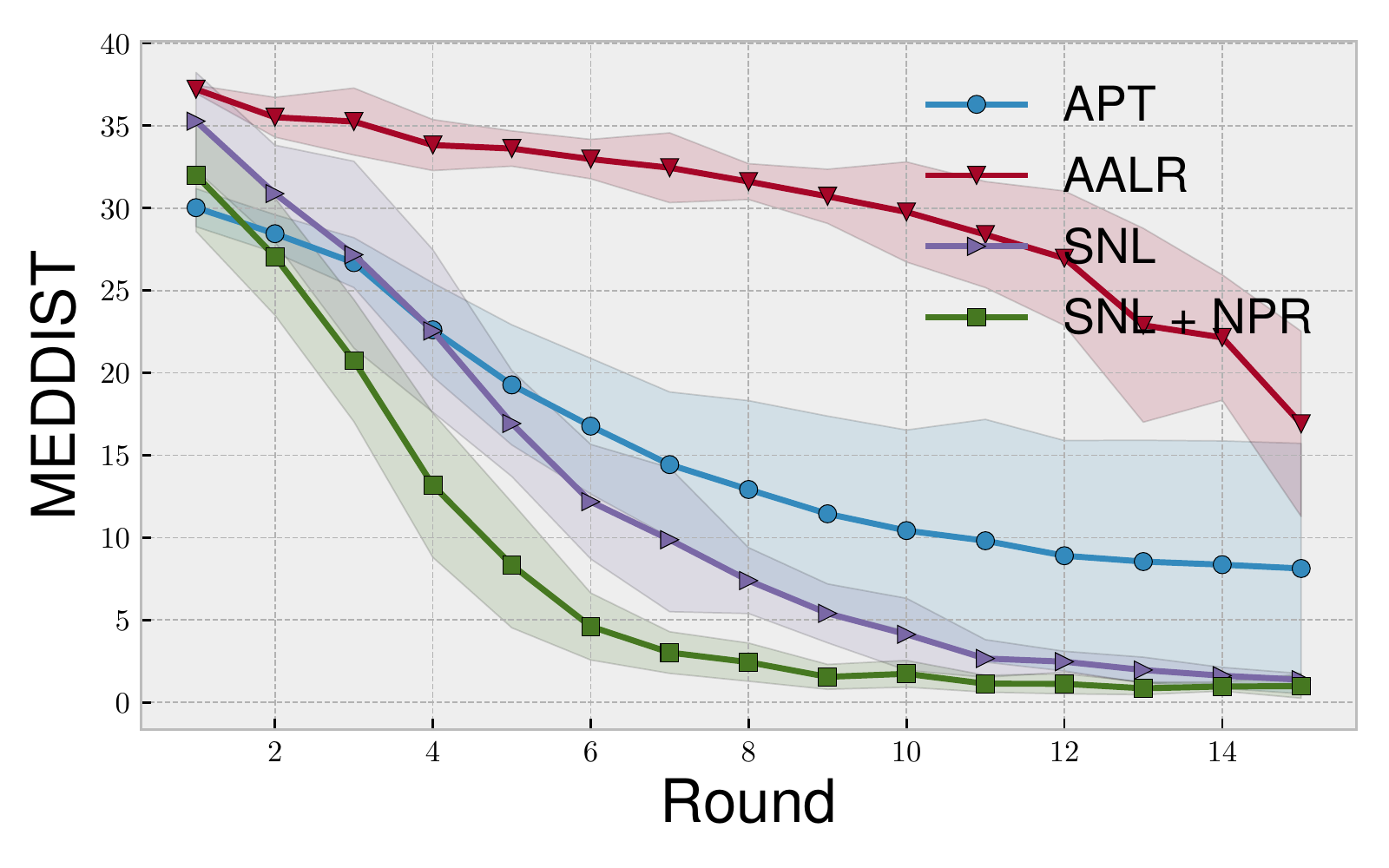}}
    \centering
    \subfigure[NLTP]{\includegraphics[width=.48\linewidth]{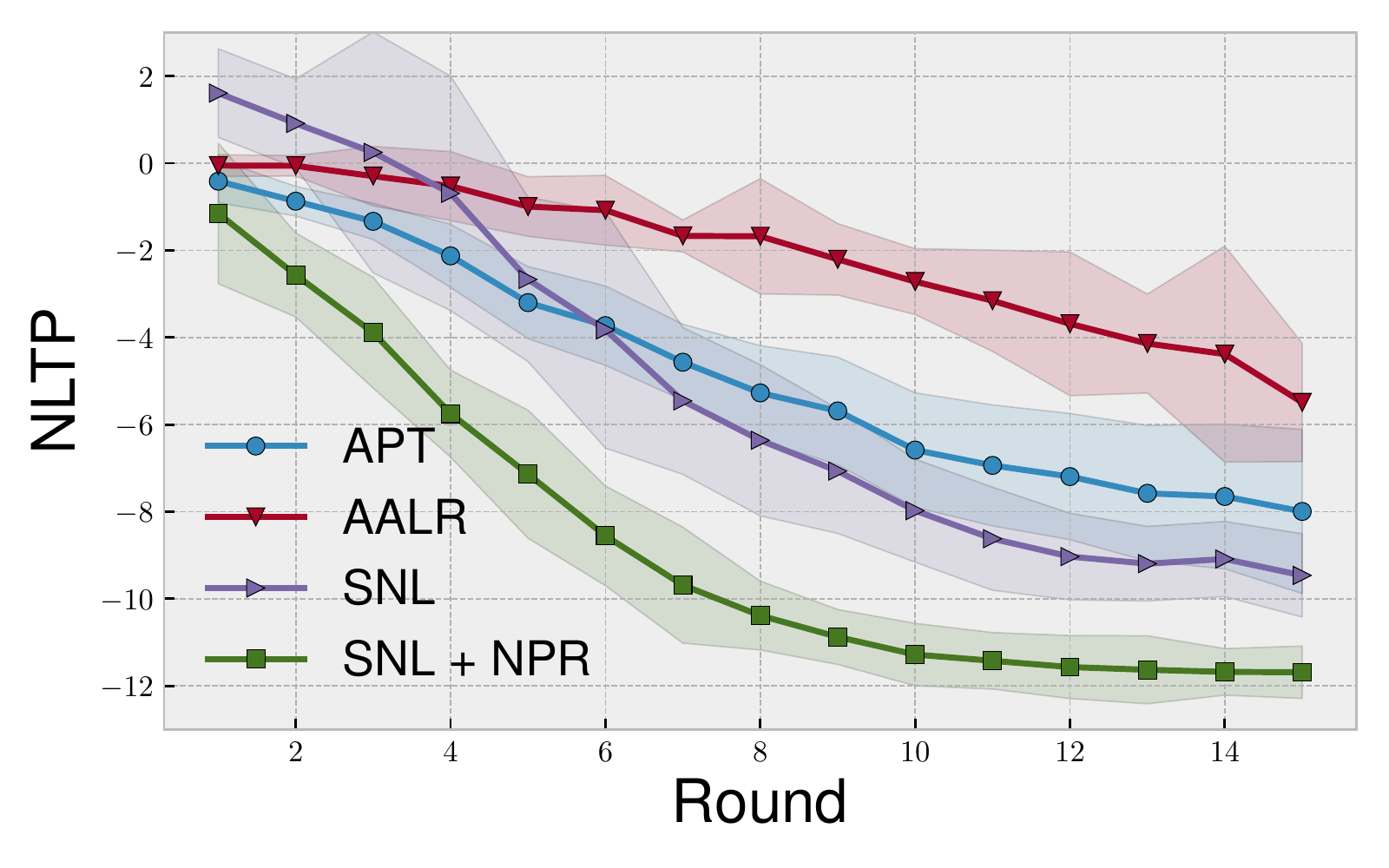}}
    \vskip -0.05in
    \caption{MEDDIST and NLTP on the MNIST experiment.}
    \label{fig:MNIST_meddist}
    \vskip -0.05in
\end{figure}

\begin{figure*}[t]
    \centering
    \includegraphics[width=0.9\linewidth]{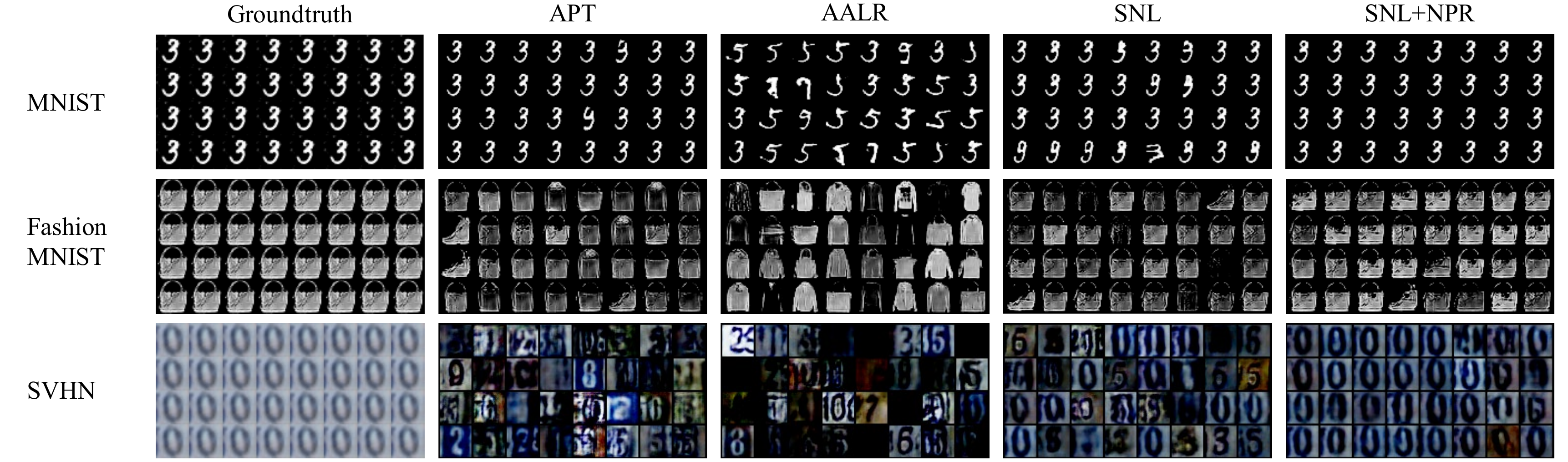}
    \caption{Image restoration based on the given groundtruth image.}
    \label{fig:image_generation}
    \vskip -0.05in
\end{figure*}

\begin{figure}[t]
\vskip -0.05in
    \centering
    \subfigure[SNL]{\includegraphics[width=.45\linewidth]{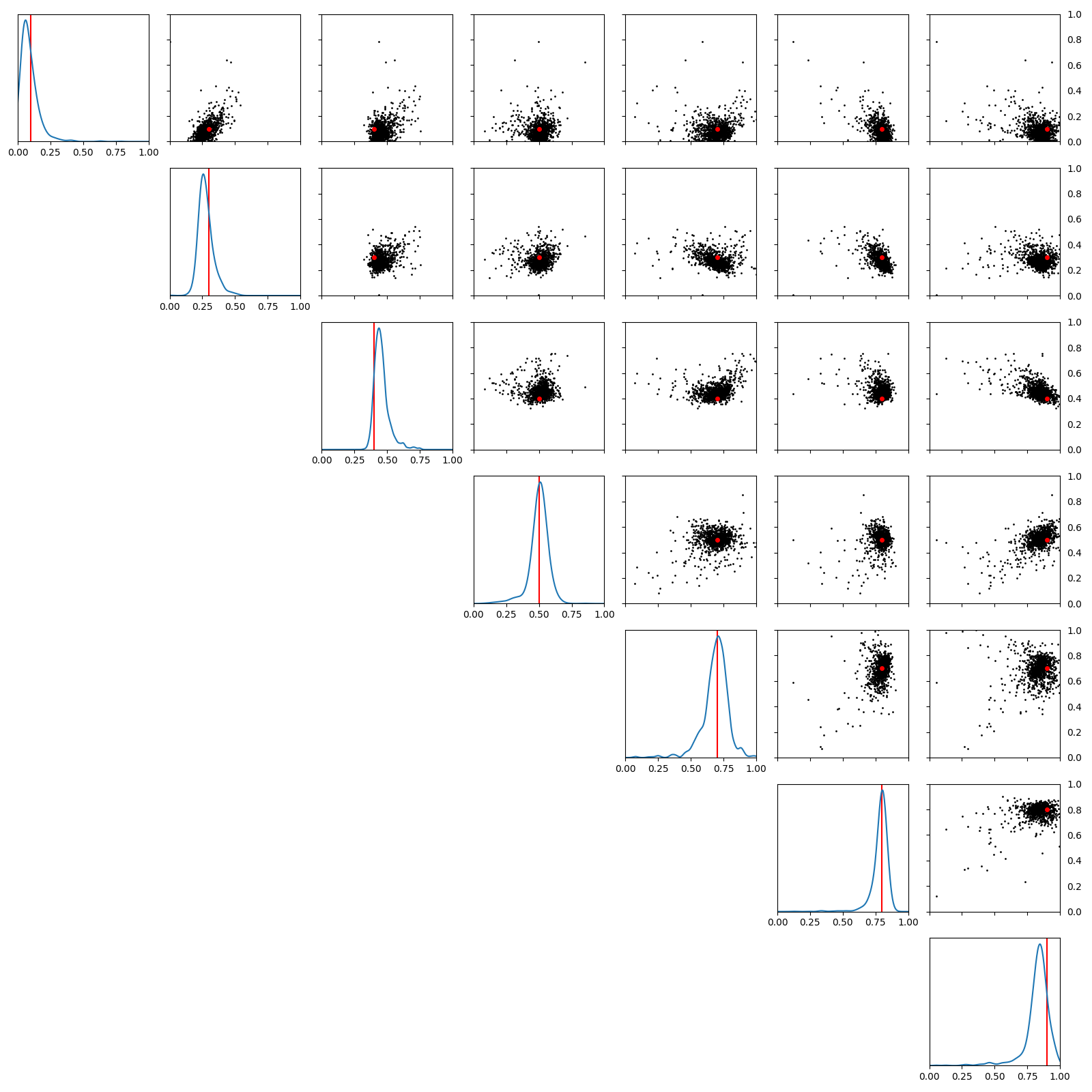}}
    \centering
    \subfigure[SNL+NPR]{\includegraphics[width=.45\linewidth]{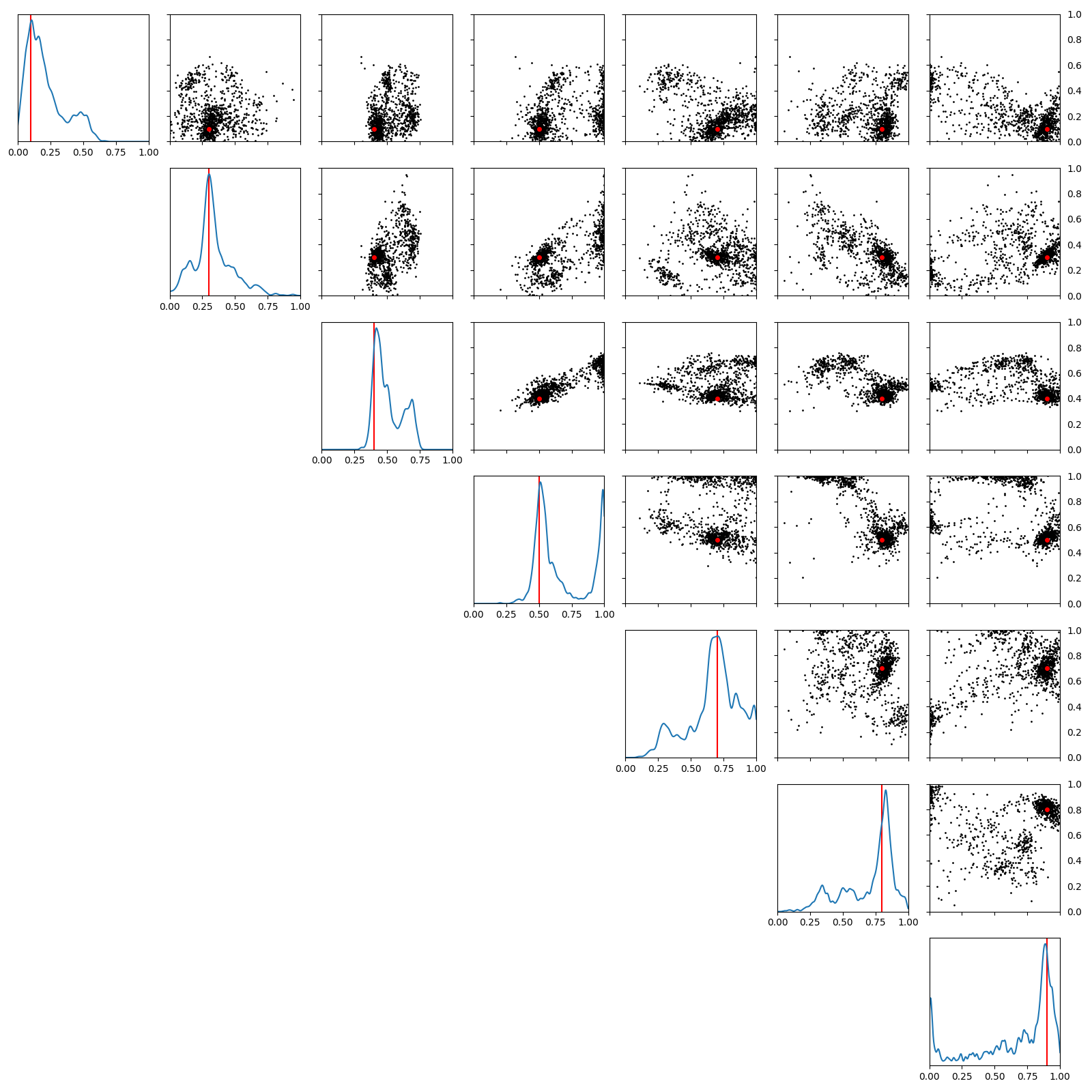}}
    \vskip -0.05in
    \caption{Approximate posterior on Fashion MNIST.}
    \label{fig:FashionMNIST_posterior}
    \vskip -0.05in
\end{figure}

\section{Experiments}\label{sec:Experiments}

\subsection{Experimental Setup}

We experiment with the regularization on a couple of tractable simulations, three science-driven simulation models, and three pre-trained GAN generator models. For the tractable simulations, we utilize the models of the SLCP-family \cite{papamakarios2019sequential}, i.e., SLCP-16/256 \cite{kim2020sequential}, to test the regularization with highly multi-modal posterior. We use the default setting of SLCP-16/256, except for simulation output dimensions, by doubling from 50 to 100 on SLCP-16 and 40 to 80 on SLCP-256. Next, we experiment with realistic simulations without the tractable likelihood. The simulations are 1) the M/G/1 model \citep{papamakarios2019sequential} from the queuing theory; 2) the Ricker model \citep{gutmann2016bayesian} from the field of ecology; and 3) the Poisson model \citep{kim2020adversarial} in the field of physics. We experiment with these simulations with three output dimension variations to test the robustness in high-dimensional cases (see Table \ref{tab:performance_real}). Other than the dimension, we comply with the original setup. Lastly, we experiment with the pre-trained GAN generator \citep{arjovsky2017wasserstein} for the image restoration task. We test on MNIST \citep{lecun1998gradient}, Fashion MNIST \citep{xiao2017fashion}, and SVHN \citep{netzer2011reading}.

We use the Neural Spline Flow \cite{durkan2019neural} for modeling both neural likelihood and neural posterior. We assume that the simulation budget per round is 100 for all but M/G/1 with 20. We run 50 rounds of inference for SLCP-256, 30 for SLCP-16, M/G/1, Ricker, and Poisson, and 15 for GAN generator models. We compare the regularized SNL with SMC-ABC \citep{sisson2007sequential}, SNPE-A \citep{papamakarios2016fast}, SNPE-B \citep{lueckmann2018flexible}, APT \citep{greenberg2019automatic}, AALR \citep{hermans2019likelihood}, and SNL \citep{papamakarios2019sequential}. For the performance metrics, we use Maximum Mean Discrepancy (MMD) \citep{sriperumbudur2010hilbert}, IS, Median distance (MEDDIST) \citep{lueckmann2021benchmarking}, Simulation-Based Calibration (SBC) \citep{papamakarios2019sequential}, as well as Negative Log-likelihood of True Parameters (NLTP) \citep{lueckmann2021benchmarking}. We release the code at \url{https://github.com/Kim-Dongjun/Neural_Posterior_Regularization}.

\subsection{Experimental Result}

Tables \ref{tab:performance} and \ref{tab:performance_real} present the quantitative results of each simulation with 30 replications.  $N=100$ for all but M/G/1 with $N=20$, and run the 50 inference rounds. The regularized algorithm shows robustness and finds most modes compared to the baselines. Fig. \ref{fig:marginal_posterior} presents the approximate posterior. It shows that the regularized SNL performs the best out of baselines with a limited simulation budget. Empirically, with $N=100$, the regularized SNL requires nearly 20 rounds to capture all modes.

Figures \ref{fig:SLCP-16} and \ref{fig:SLCP-256} illustrate the performance by rounds. The regularized SNL consistently outperforms the baselines in terms of the IS. In particular, the regularized SNL could be framed as a mix of APT and SNL in its loss design, but the regularized SNL outperforms both APT and SNL. Fig. \ref{fig:SLCP-16}-(b) and Fig. \ref{fig:SLCP-256}-(b) empirically demonstrate that the regularization gives robust inference across diverse dimensions.

Figure \ref{fig:image_generation} shows the image restoration from a single shot of the given image on MNIST, Fashion MNIST, and SVHN. The regularized SNL significantly outperforms the baselines on all tasks regarding the generated sample quality. In contrast to SNL, the regularized SNL finds multiple modes in Fig. \ref{fig:FashionMNIST_posterior}. Quantitatively, we compare the regularized SNL with baselines in Fig. \ref{fig:MNIST_sbc} and Fig. \ref{fig:MNIST_meddist} on the MNIST experiment.

\section{Conclusion}\label{sec:Conclusion}
This paper proposes a new regularization for \textit{likelihood-free inference}. We approximate this regularization as NPR, and the regularized SNL can be interpreted as the joint combination of APT and SNL in a unified framework. The optimality of the regularized SNL is driven as a closed-form solution, and the tuning of the regularization magnitude does not require additional cost. The experimental results support that the proposed regularization method takes benefits from both SNL and APT.

\bibliographystyle{elsarticle-num-names}
\bibliography{reference.bib}

\end{document}